\documentclass{article}

\usepackage{hyperref,graphicx,amsmath,amsfonts,amssymb,bm,url,breakurl,epsfig,epsf,color,fullpage,MnSymbol,mathbbol,fmtcount,semtrans,cite,caption,subcaption,multirow,comment}
\usepackage[noend]{algpseudocode}

\usepackage{amssymb}

\usepackage[utf8]{inputenc} 
\usepackage[T1]{fontenc}    
\usepackage{url}            
\usepackage{booktabs}       
\usepackage{amsfonts}       
\usepackage{nicefrac}       
\usepackage{microtype}      

\makeatletter
\providecommand*{\boxast}{%
  \mathbin{
    \mathpalette\@boxit{*}%
  }%
}
\newcommand*{\@boxit}[2]{%
  \sbox0{$\m@th#1\Box$}%
  \ifx#1\displaystyle \ht0=\dimexpr\ht0+.05ex\relax \fi
  \ifx#1\textstyle \ht0=\dimexpr\ht0+.05ex\relax \fi
  \ifx#1\scriptstyle \ht0=\dimexpr\ht0+.04ex\relax \fi
  \ifx#1\scriptscriptstyle \ht0=\dimexpr\ht0+.065ex\relax \fi
  \sbox2{$#1\vcenter{}$}
  \rlap{%
    \hbox to \wd0{%
      \hfill
      \raisebox{%
        \dimexpr.5\dimexpr\ht0+\dp0\relax-\ht2\relax
      }{$\m@th#1#2$}%
      \hfill
    }%
  }%
  \Box
}
\makeatother

  \makeatletter
\def\BState{\State\hskip-\ALG@thistlm}
\makeatother

  \usepackage{mathtools}

\usepackage{titlesec}

\usepackage{tikz}
\usepackage{pgfplots}
\usetikzlibrary{pgfplots.groupplots}

\titleformat{\paragraph}
{\normalfont\normalsize\bfseries}{\theparagraph}{1em}{}
\titlespacing*{\paragraph}
{0pt}{3.25ex plus 1ex minus .2ex}{1.5ex plus .2ex}

\usepackage{movie15}

\usepackage{cite}
\usepackage{caption}
\usepackage[bottom,hang,flushmargin]{footmisc} 

\newcommand{\tsn}[1]{{\left\vert\kern-0.25ex\left\vert\kern-0.25ex\left\vert #1 
    \right\vert\kern-0.25ex\right\vert\kern-0.25ex\right\vert}}

\definecolor{darkred}{RGB}{150,0,0}
\newcommand{\distas}{\overset{\text{i.i.d.}}{\sim}}
\definecolor{darkgreen}{RGB}{0,150,0}
\definecolor{darkblue}{RGB}{0,0,200}
\hypersetup{colorlinks=true, linkcolor=darkred, citecolor=darkblue, urlcolor=darkblue}

\newtheorem{theorem}{Theorem}[section]

\newtheorem{assumption}{Assumption}

\newtheorem{lemma}[theorem]{Lemma}

\newtheorem{definition}[theorem]{Definition}

\newcommand{\eps}{\varepsilon}
\newcommand{\epst}{\tilde{\varepsilon}}

\newcommand{\vs}{\vspace{-0.0cm}}
\newcommand{\remove}[1]{}

\newcommand{\bXi}{{\bar{\Xi}}}

\newcommand{\beq}{\begin{equation}}

\newcommand{\eeq}{\end{equation}}

\newcommand{\nn}{\nonumber}
\newcommand{\la}{\lambda}

\newcommand{\where}{\quad\text{where}\quad}

\newcommand{\nt}{{u}}

\newcommand{\Lc}{{\cal{L}}}
\newcommand{\Uc}{{\cal{U}}}
\newcommand{\tang}[1]{\text{cot}(#1)}
\newcommand{\bti}{\bt_{\text{init}}}

\newcommand{\bSi}{{\boldsymbol{{\Sigma}}}}

\newcommand{\bmu}{{\boldsymbol{{\mu}}}}

\newcommand{\Iden}{{\mtx{I}}}

\newcommand{\Gm}{\bar{\Gamma}}

\newcommand{\z}{{\vct{z}}}

\newcommand{\tn}[1]{\|{#1}\|_{\ell_2}}

\newcommand{\tsub}[1]{\|{#1}\|_{\psi_2}}

\newcommand{\Ac}{\mathcal{A}}
\newcommand{\Dc}{\mathcal{D}}
\newcommand{\Rc}{\mathcal{R}}

\newcommand{\bt}{{\boldsymbol{\beta}}}

\newcommand{\bth}{{\boldsymbol{\hat{\beta}}}}

\newcommand{\Sc}{\mathcal{S}}

\newcommand{\Sct}{\mathcal{U}}

\newcommand{\Lch}{{\cal{L}_{\Sc}}}
\newcommand{\tell}{{\tilde{\ell}}}
\newcommand{\Lct}{{\tilde{\cal{L}}}}
\newcommand{\Lcth}{{\tilde{\cal{L}}_{\Uc}}}

\newcommand{\Nn}{\mathcal{N}}

\newcommand{\vb}{\vct{v}}

\newcommand{\abb}{\mtx{\bar{a}}}

\newcommand{\li}{\left<}
\newcommand{\ri}{\right>}
\newcommand{\s}{\vct{s}}
\newcommand{\ab}{\vct{a}}
\newcommand{\bb}{\vct{b}}

\newcommand{\h}{\vct{h}}

\newcommand{\g}{{\vct{g}}}

\newcommand{\corr}[1]{{\rho}(#1)}

\newcommand{\Fc}{\mathcal{F}}

\newcommand{\xb}{\bar{\x}}
\newcommand{\xp}{\x'}
\newcommand{\yh}{\hat{y}}

\newcommand{\xt}{\tilde{\x}}
\newcommand{\zt}{\tilde{\z}}
\newcommand{\yt}{\tilde{\y}}

\newcommand{\st}{\star}

\newcommand{\cA}{\mathcal{A}}

\newcommand{\x}{\vct{x}}

\newcommand{\y}{\vct{y}}

\newcommand{\bgl}{{~\big |~}}

\definecolor{emmanuel}{RGB}{255,127,0}

\newcommand{\R}{\mathbb{R}}
\newcommand{\Pro}{\mathbb{P}}
\newcommand{\conv}{\overset{\Pro}{\rightarrow}}

\newcommand{\sgn}[1]{\textrm{sgn}(#1)}

\newcommand{\E}{\operatorname{\mathbb{E}}}

\newcommand{\e}{\mathrm{e}}

\newcommand{\vct}[1]{\bm{#1}}
\newcommand{\mtx}[1]{\bm{#1}}

\numberwithin{equation}{section} 

\def \endprf{\hfill {\vrule height6pt width6pt depth0pt}\medskip}
\newenvironment{proof}{\noindent {\bf Proof} }{\endprf\par}

\title{Statistical and Algorithmic Insights for\\Semi-supervised Learning with Self-training}

\author{%
  Samet Oymak\thanks{Email: \texttt{oymak@ece.ucr.edu}. Department of Electrical and Computer Engineering, University of California, Riverside.}\quad\quad\quad\quad\quad \quad\quad\quad\quad\quad Talha Cihad Gulcu\thanks{Email: \texttt{tcgulcu@gmail.com}.}
}
\begin{document}

\maketitle

\begin{abstract}
Self-training is a classical approach in semi-supervised learning which is successfully applied to a variety of machine learning problems. Self-training algorithm generates pseudo-labels for the unlabeled examples and progressively refines these pseudo-labels which hopefully coincides with the actual labels. This work provides theoretical insights into self-training algorithm with a focus on linear classifiers. We first investigate Gaussian mixture models and provide a sharp non-asymptotic finite-sample characterization of the self-training iterations. Our analysis reveals the provable benefits of rejecting samples with low confidence and demonstrates that self-training iterations gracefully improve the model accuracy even if they do get stuck in sub-optimal fixed points. We then demonstrate that regularization and class margin (i.e.~separation) is provably important for the success and lack of regularization may prevent self-training from identifying the core features in the data. Finally, we discuss statistical aspects of empirical risk minimization with self-training for general distributions. We show how a purely unsupervised notion of generalization based on self-training based clustering can be formalized based on cluster margin. We then establish a connection between self-training based semi-supervision and the more general problem of learning with heterogenous data and weak supervision.


\end{abstract}

\vs\section{Introduction}\vs\vs

The recent widespread success of deep neural networks rely on the presence of large labeled datasets to a significant extent. Unfortunately, such good-quality datasets may not be readily available for variety of practical applications. Indeed, a grand challenge in {{expanding machine learning to new domains}} is the cost of obtaining good quality labels. This is especially true for privacy and safety sensitive tasks that are abundant in critical domains such as healthcare and defense. On the other hand, unlabeled data can be relatively cheap to obtain and may be more abundant. This necessitates semi/unsupervised learning algorithms that can go beyond supervised learning and efficiently utilize unlabeled data.

Semi-supervised learning (SSL) techniques aim to reduce the dependence on the labeled data by making use of unlabeled data. A large number of approaches for SSL involve an extra loss term accounting for unlabeled data which is expected to help the model better generalize to unseen data. Self-training, consistency training and entropy minimization are among some of the core methods (discussed in Section~\ref{prior_art} in more detail) used for the purpose of SSL. Despite its popularity and practical success, we still don't have a fundamental understanding of when and why self-training algorithms work. For instance, self-training algorithms gradually utilizes unlabeled data by first incorporating the most reliable pseudo-labels. Are there setups where rejecting unreliable examples provably help? Similarly, generating and overfitting to incorrect pseudo-labels is a natural concern in SSL. On the other hand, recent empirical and theory literature suggests that, for supervised learning, interpolating to training data performs surprisingly well even when the model perfectly interpolates and achieves zero training loss \cite{belkin2019reconciling,hastie2019surprises,zhang2016understanding}. How crucial is regularization when it comes to learning with unlabeled data? Finally, for which datasets, self-training finds useful models that generalize better and what structural assumptions on the data are key to success?

\vspace{5pt}
\noindent{\bf{Contributions.}} This paper takes a step towards addressing the aforementioned questions by studying algorithmic fundamentals of SSL. Specifically, we make the following contributions.

$\bullet$ {\bf{Self-training for Gaussian Mixture Models:}} One way to understand the algorithmic performance is by focusing on fundamental dataset models such as Gaussian mixtures and conducting a careful analysis capturing exact algorithmic performance. We study the problem of learning a linear classifier with self-training under a Gaussian mixture model (GMM). We precisely calculate the distributional properties of self-training iterations. Specifically we capture the evolution of the correlation between the optimal classifier and the self-training output in a non-asymptotic fashion. This reveals (non)-asymptotic formulae exactly characterizing the performance of self-training with linear models. We present \remove{scenarios and} associated numerical experiments demonstrating the classification performances under various scenarios which also reveals the provable benefits of rejecting weak examples.

$\bullet$ {\bf{Algorithmic Insights: The Role of Distribution and Regularization:}} Next, we explore the importance of distributional properties by considering a more general family of mixture models where the means of mixture components are continuously distributed. This reveals that as long as there is a margin (i.e.~separation) between the means, unlabeled data improves the performance, however without margin, un-regularized algorithm provably gets stuck under least-squares loss. We then show how ridge regularization and early stopping can mitigate this issue by encouraging self-training to pick up the principal eigendirections in the data in a similar fashion to power iteration. We also discuss similar benefits of regularization for logistic regression.

$\bullet$ {\bf{Statistical Insights: Empirical Risk Minimization with Self-Training:}} Focusing on general data distributions, we consider ERM with self-training. When the problem is purely unsupervised, we discuss how an unsupervised notion of generalization can be formalized based on the margin induced by the clusters found by self-training. Secondly, we discuss the loss landscapes of the supervised and unsupervised components of self-training.  Inspired from the seminal results of \cite{balcan2010discriminative}, we connect self-training based semi-supervised learning to the more general problem of learning with heterogenous datasets and formalize how unlabeled and labeled data can be viewed as weak-supervision and strong supervision respectively.


\remove{
For least-squares loss, we demonstrate that a notion of margin is critical for guiding pseudo-labeling towards useful solutions. We also show distributional bias can be encouraged by ridge regularization (or early stopping) to overcome the margin requirement.
}
\remove{
\item {\bf{Learning with distribution mismatch:}} We provide generalization bounds for heterogeneous loss functions that involve two distributions which can then be specialized to pseudo-labeling approach. These bounds show that weak distribution can help greatly narrow down the search space for the optimal model. 
}

\subsection{Prior Art}\vs
\label{prior_art}
The benefits of using unlabeled data for learning models is subject of a rich literature since 70s which consider a variety of settings such as generative models \cite{castelli1995exponential,nigam2000text},
semi-supervised support vector machines \cite{vapnik1998statistical,joachims1999transductive}, graph-based models 
\cite{blum2001learning,belkin2006manifold,zhu2003semi}, or co-training \cite{blum1998combining} and multiview models\cite{sindhwani2005co}. The relative value of labeled and unlabeled samples in a detection-estimation theoretical framework is examined in \cite{castelli1996relative}.
A line of work is related to how the presence of unlabeled data be useful to limit Radamacher complexity\cite{bartlett2002rademacher}. For example,
the compatibility of a target function with respect to a data distribution is considered by \cite{balcan2010discriminative}, where the authors illustrate how enough unlabeled data can be useful to reduce the size of the search space.
It is demonstrated by several papers \cite{oneto2011impact,oneto2015local,oneto2016global} that the additional unlabeled data can be used to improve the tightness of the Radamacher complexity (RC) based bounds. 
A sharper generalization error bound for multi-class learning with the help of additional unlabeled data is presented by \cite{li2019multi}, along with an efficient multi-class classification algorithm using local Radamacher complexity and unlabeled samples.
Apart from that, semi supervised learning (SSL) is a versatile approach for training models without using a large amount of data. 
SSL algorithms can achieve performance improvement with low cost, and there are a large number of SSL methods \cite{miyato2018virtual,sajjadi2016regularization,laine2016temporal,tarvainen2017mean,berthelot2019mixmatch,xie2019unsupervised,berthelot2020remixmatch,lee2013pseudo,sajjadi2016mutual} available in the literature.

A large portion of SSL methods relies on generating an artificial label for unlabeled data and
training the model to predict those artificial labels when
the unlabeled data is used as the input. Pseudo-labeling \cite{lee2013pseudo} is one of such methods where the class prediction of the model is used for training purposes. Consistency regularization is also an important component of many SSL algorithms. Consistency regularization \cite{tarvainen2017mean,sajjadi2016regularization,laine2016temporal} is based on the approach that the model is supposed to 
generate similar outputs when perturbed version of the same data is applied as the input. Adversarial transformation is used by \cite{miyato2018virtual} in the loss function of consistency training,
and cross-entropy loss instead of squared loss function 
appears in the works \cite{miyato2018virtual,xie2019unsupervised}.
There are also hybrid algorithms combining diverse mechanisms. For example, Fix-Match \cite{sohn2020fixmatch} combines pseudo-labeling and consistency training to generate artificial labels. Mix-Match \cite{berthelot2019mixmatch}, ReMixMatch \cite{berthelot2020remixmatch}, unsupervised data augmentation \cite{xie2019unsupervised} are among other composite approaches.
Self training in the setting of domain adaptation is covered by the papers \cite{long2013transfer,inoue2018cross}. Class balance \cite{zou2018unsupervised} and confidence regularization \cite{zou2019confidence} for self-training are among other lines of works. Gradual domain adaptation in regularized models is analyzed by \cite{kumar2020understanding}.  
The papers \cite{carmon2019unlabeled,zhai2019adversarially,najafi2019robustness,stanforth2019labels} show theoretically and empirically how semi-supervised learning procedure can achieve high robust accuracy and improve adversarial robustness.

\vs\section{Problem setup}\vs
Let us first fix the notation. Given an event $E$, let $1(E)$ be the indicator function of $E$ which is $1$ if $E$ happens and $0$ otherwise. We use $X\bgl E$ to denote the conditional random variable induced by a random variable $X$ given an event $E$. We will refer the vectors with unit Euclidean norm as unit norm. Given two vectors $\ab,\bb$, their correlation is denoted by $\rho(\ab,\bb)=\frac{\li\ab,\bb\ri}{\tn{\ab}\tn{\bb}}$. Related to correlation, we define {\em{co-tangent of the angle between two vectors}} to be
\[
\tang{\ab,\bb}=\frac{\rho(\ab,\bb)}{\sqrt{1-\rho(\ab,\bb)^2}},
\]
which will be useful for cleaner notation. As $\tang{\ab,\bb}\rightarrow\infty$, the two vectors become perfectly correlated i.e.~$\rho(\ab,\bb)\rightarrow 1$. Let $Q(\cdot)$ be the tail of a standard normal variable and $Q_X$ be the tail of the distribution of a random variable $X$. $\conv$ denotes convergence in probability. $a\wedge b$ and $a\vee b$ returns minimum and maximum of two scalars. Finally, $(a)_+$ returns $a\vee 0$.


Let $\Sc=(y_i,\x_i)_{i=1}^n\in\{-1,1\}\times \R^p$ be independent and identically distributed (i.i.d) labeled sampled distributed as $\Dc=\Dc_{y| \x}\times \Dc_{\x}$ and let $\Sct=(\x_i)_{i=n+1}^{n+\nt}$ be i.i.d.~unlabeled samples distributed with the marginal distribution $\Dc_{\x}$. Let $f:\R^p\rightarrow\R$ be a prediction function (e.g.~a neural network) and let $\yh_f(\x)$ be the hard-label ($-1,1$) assigned to $f(\x)$ defined as
\[
\yh_f(\x)=\begin{cases}1\quad~~\text{if}\quad f(\x)\geq 0\\-1\quad \text{else}\end{cases}.
\]
The standard self-training approach is sufficiently general to operate on a generic algorithm. The algorithm can self-train by using its own labels $\yh_f(\x)$ which are also known as pseudo-labels. Self-training is often gradual, it first utilizes examples where predictions are confident and only later moves to examples which are less certain. Thus, it is a common strategy to reject weak pseudo-labels and use the more confident ones. Given a loss function $\ell$, function class $\Fc$, and acceptance threshold $\Gamma\geq 0$, self-training with pseudo-labels typically solves an empirical risk minimization problem of the form
\begin{align}
\hat f=\arg\min_{f\in \Fc} \underbrace{\frac{1}{n}\sum_{i=1}^n\ell(y_i,f(\x_i))}_{\Lch(f)}+\la\underbrace{\frac{1}{\nt}\sum_{i=n+1}^{n+\nt}1(|f(\x_i)|\geq\Gamma)\ell(\yh_f(\x_i),f(\x_i))}_{\Lcth(f)}.\label{PL reg}
\end{align}
where $\Lch$ and $\Lcth$ are the supervised and unsupervised empirical risks respectively. Let us also introduce our iterative learning setup. Suppose we have an algorithm $\cA$ that takes a labeled dataset and builds a prediction model $f$. An obvious example for $\Ac$ is \eqref{PL reg}. Denote the initial model by $f_0$ and let $\Gamma\geq 0$ be the acceptance threshold. Given a stopping time $T$, the self-training algorithm we consider operates in two steps.

$\bullet$ {\bf{Step 1: Create Pseudo-labels:}} From $\Uc$ and current iterate $f_\tau$, determine a subset $\Uc_\tau=(\xt_i,\yt_i)$ where $\xt_i\in\Uc$ are the acceptable inputs that satisfy $|f_\tau(\xt_i)|\geq \Gamma$ and $\yt_i$ are the pseudo-labels $\yt_i=\yh_{f_\tau}(\xt_i)$.

$\bullet$ {\bf{Step 2: Refine the model:}} Obtain the new classifier via $f_{\tau+1}=\cA(\Sc, \Uc_\tau)$. If $\tau< T$, go to Step 1.

We remark that $\cA$ can treat the datasets $\Sc$ and $\Uc_\tau$ differently in a similar fashion to \eqref{PL reg}, e.g.~by weighting labeled $\Sc$ higher than pseudo-labeled $\Uc$. In our analysis of iterative algorithms in Sections \ref{sec gauss} and \ref{sec algo}, we consider a slightly different version where we only use the unlabeled data for refinement in Step 2. While our approach does extend to jointly learning over $(\Sc,\Uc)$, as we shall see, learning only over $\Uc$ results in cleaner and more insightful bounds.

\vs\section{Understanding Self-Training for Mixtures of Two Gaussians}\label{sec gauss}
\begin{figure}[t!]
	\begin{subfigure}{2.2in}
		\includegraphics[scale=0.26]{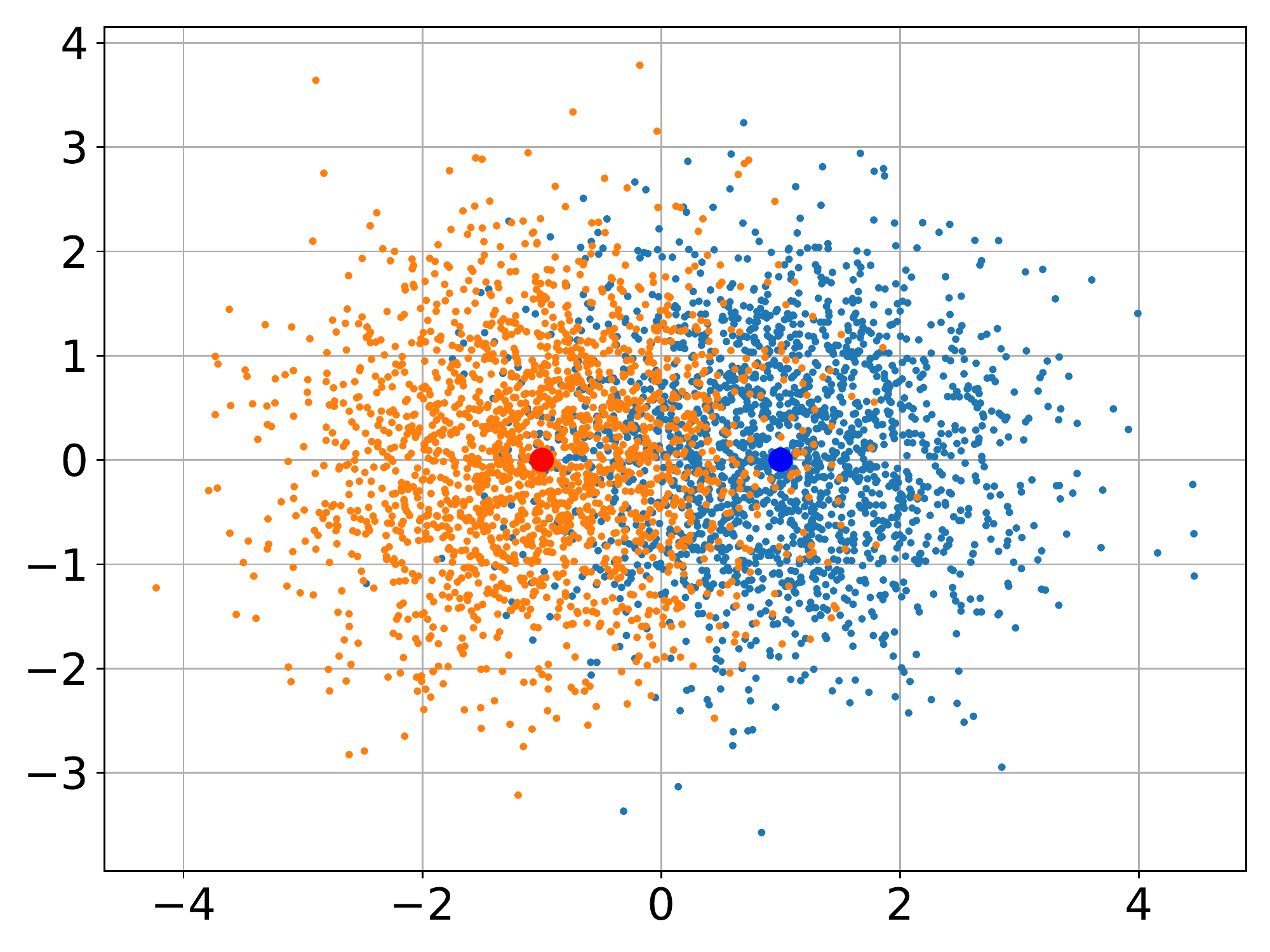}\caption{$\Gamma=0$}\label{fig1a}
	\end{subfigure}
	\begin{subfigure}{2.2in}
		\includegraphics[scale=0.26]{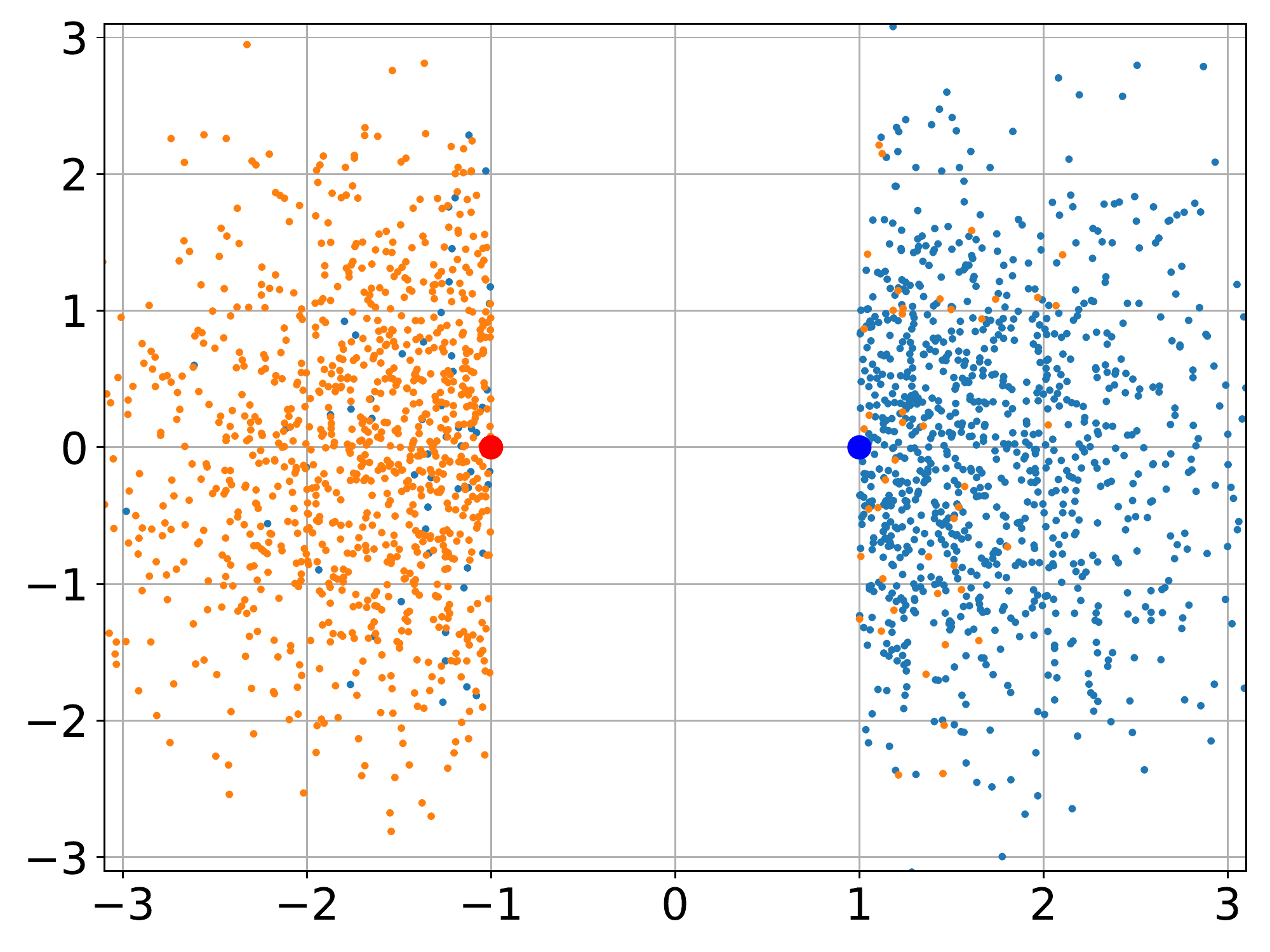}\caption{$\Gamma=1$, $\rho(\bti,\bmu)=1$}\label{fig2a}
	\end{subfigure}
	\begin{subfigure}{2.2in}
		\includegraphics[scale=0.26]{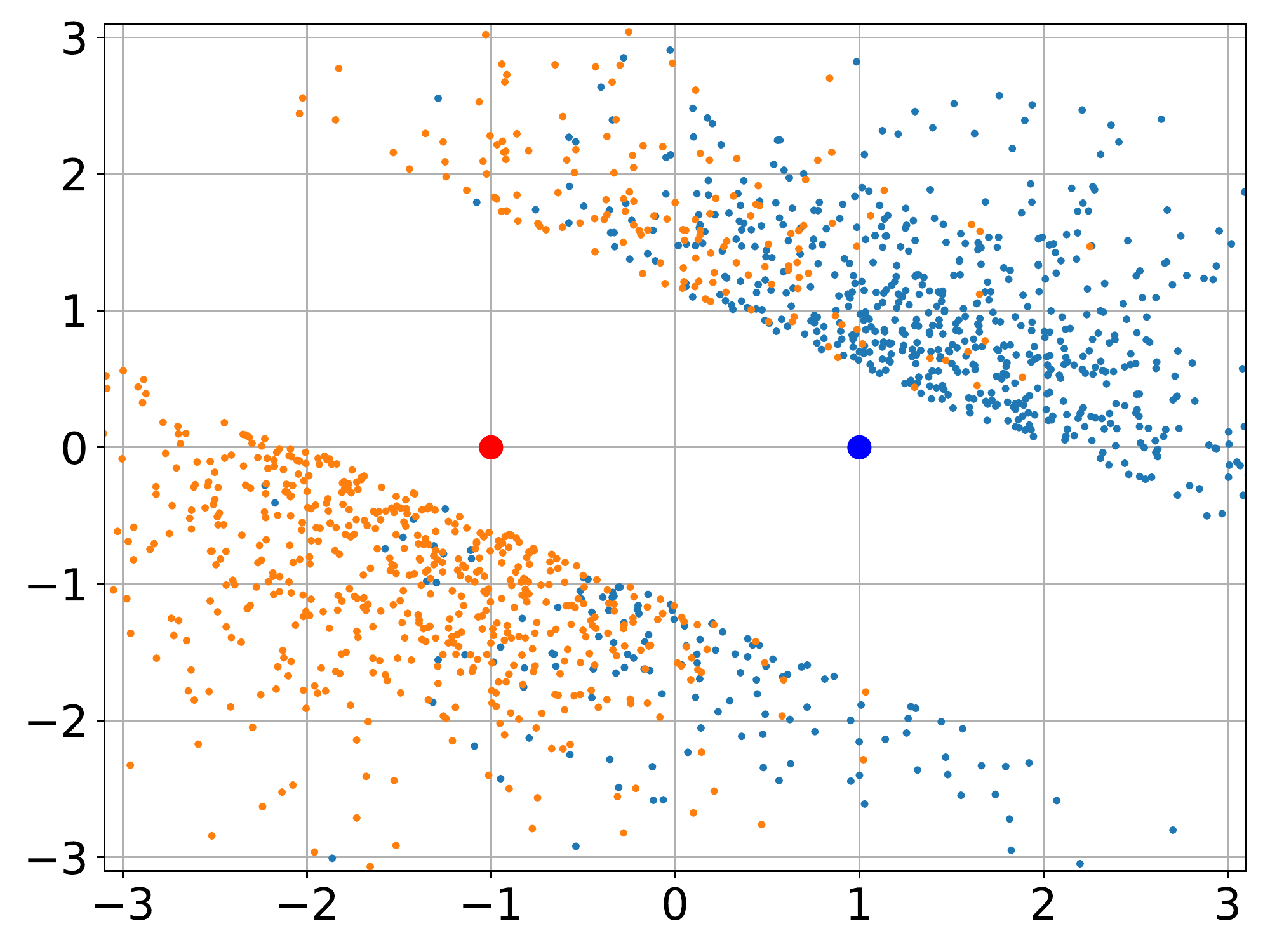}\caption{$\Gamma=1$, $\rho(\bti,\bmu)=0.5$}\label{fig3a}
	\end{subfigure}
	\caption{Visualization of a Binary GMM with noise variance $\sigma^2=1$. Sample size is 4000. The large dots at -1 and 1 are the mixture centers $\pm\bmu=[\pm 1,0]$. Acceptance threshold $\Gamma$ removes the examples with low-correlation to the initial model $\bti$.}
	\label{fig:accept} \vspace{-0.2cm}
\end{figure}
We start with a definition of the distribution we will study.
\begin{definition}[Binary Gaussian Mixture Model (GMM)] \label{GMM def}The distribution $(\x,y)\sim \Dc$ is given as follows. Fix a unit vector $\bmu\in\R^p$ and scalar $\sigma\geq 0$. Let $y$ be a Rademacher random variable ($\Pro(y=1)=1-\Pro(y=-1)=1/2$) and $\x\sim \Nn(y\bmu,\sigma^2 \Iden_p)$.
\end{definition}
Note that the component mean $\bmu$ is also the optimal linear classifier. If we have labeled data $\Sc=(\x_i,y_i)_{i=1}^n$, $\bmu$ can be estimated via the 
\begin{align}
\text{{{Averaging estimator}}}~\bti=\frac{1}{n}\sum_{i=1}^n y_i\x_i.\label{supervised model}
\end{align}
This estimator also coincide with the ridge regularized least-squares (e.g.~$\arg\min_{\bt}\sum_{i=1}^n (y_i-\x_i^T\bt)^2+\la\tn{\bt}^2$) when the regularization parameter $\la\rightarrow \infty$. Perhaps surprisingly, this estimator is known to be the Bayes optimal classifier for GMM if we have access to the labeled data alone \cite{mignacco2020role, lelarge2019asymptotic}. This motivates us to investigate the analytical properties of the averaging estimator by adapting it to self-training as explained earlier. Given this initial supervised model $\bti$ and the unlabeled dataset $\Uc=(\x_i)_{i=1}^u$ sampled from GMM, we consider the pseudo-label estimator
\begin{equation}
\bth=\text{self-train}(\bti,\Uc)\where\text{self-train}(\bti,\Uc)= \frac{\sum_{i=1}^{u} 1(|\frac{\bti^T\x_i}{\tn{\bti}}|\geq \Gamma)\sgn{\bti^T\x_i}\x_i}{\sum_{i=1}^{u} 1(|\frac{\bti^T\x_i}{\tn{\bti}}|\geq \Gamma)}.
\label{pseudo_label_est}    
\end{equation}
where $\Gamma\geq 0$ is the acceptance threshold eliminating low-confidence predictions. Acceptance threshold is commonly used in practical semi-supervised learning approaches \cite{xie2019unsupervised,mcclosky2006effective,yarowsky1995unsupervised}. The impact of acceptance threshold is illustrated in Figure \ref{fig:accept} where points are projected on two dimensions. Here the mixture center $\bmu$ is the $[1~0~0~\dots~0]$ direction. When $\Gamma=0$, we accept all points which corresponds to a Binary GMM distribution. When $\Gamma$ is non-zero, the conditional distribution of the accepted examples depend on the quality of the initial model $\bti$. Figure \ref{fig2a} and \ref{fig3a} chooses $\Gamma=1$ for different $\bti$. In Figure \ref{fig2a}, $\bti$ is aligned with $\bmu$ (correlation is $1$) which results in a clean separation between the two classes (the red and blue dots) while rejecting 50\% of the samples that lie between the mixture centers $\pm1$. In Figure \ref{fig3a}, correlation coefficient between $\bti$ and $\bmu$ is $1/2$ and $\bti$ has a higher classification error. As a result, the two classes are not as cleanly separated despite using rejection.


The following theorem provides a sharp non-asymptotic bound for the pseudo-label estimator \eqref{pseudo_label_est}. Below, we set $\gamma_p=\E_{\g\sim \Nn(0,\Iden_p)}[\tn{\g}]^2$. It is well-known that $\gamma_p$ satisfies $p\geq \gamma_p\geq p-1$. 
\begin{theorem} [Non-asymptotic Bound for GMM] \label{sharp_bound_GMM}Let $\bmu\in\R^p$ be a unit norm vector from Def.~\ref{GMM def} and suppose $\bti\in\R^p$ has correlation $\rho(\bti,\bmu)=\alpha>0$. Set $\beta=\sqrt{1-\alpha^2}$. Draw $u$ i.i.d.~unlabeled samples $(\x_i)_{i=1}^u$ from GMM. 
Let $\bth$ be defined as $\bth=\text{self-train}(\bti,(\x_i)_{i=1}^u).$
Fix resolution $1/2>\eps>0$ and absolute constant $c>0$. Define the normalized thresholds $\bar{\Gamma}_-=\frac{\alpha+\Gamma}{\sigma}$ and $\bar{\Gamma}_+=\frac{\Gamma-\alpha}{\sigma}$ and the quantities
\begin{align}
&\Lambda=\frac{1}{\sqrt{2\pi}\rho}(\e^{-\bar{\Gamma}_+^2/2}+\e^{-\bar{\Gamma}_-^2/2})\quad\text{and}\quad \rho=Q(\bar{\Gamma}_+)+Q(\bar{\Gamma}_-)\quad\text{and}\quad \nu=Q(\bar{\Gamma}_-)/\rho.
\end{align}
With probability $1-10\e^{-c\eps^2 ((p-3)\wedge \rho u)}$, we have that
\[
\frac{1+\sigma\alpha \Lambda-2\nu+(1+\sigma)\eps}{\sigma\sqrt{(\beta \Lambda-\eps)_+^2+(1-\eps)_+^2\gamma_{p-2}/u\rho}} \geq \tang{\bth,\bmu}\geq \frac{1+\sigma\alpha \Lambda-2\nu-(1+\sigma)\eps}{\sigma\sqrt{(\beta \Lambda+\eps)^2+(1+\eps)^2\gamma_{p-2}/u\rho}}.
\]
Thus, fixing $\bar{u}=u/p$ and letting $p\rightarrow \infty$, we have that
\[
\lim_{p\rightarrow\infty}\tang{\bth,\bmu}\overset{\Pro}{\rightarrow}\frac{1+\sigma\alpha \Lambda-2\nu}{\sigma\sqrt{(1-\alpha^2) \Lambda^2+1/\bar{u}\rho}}.
\]
\end{theorem}

Theorem \ref{sharp_bound_GMM} shows that pseudo-label optimization as defined by \eqref{pseudo_label_est} can be useful to obtain a higher correlation and thus can improve the quality of the initial direction $\bti$.
Let $f$ denote the transformation that is applied to $\rho(\bti,\bmu)$ as a result of pseudo-label optimization. Theorem \ref{sharp_bound_GMM} provides matching upper and lower bounds for the evolution of the co-tangent. Specifically, using the relation between correlation and co-tangent, as $p\rightarrow\infty$, we have that 
\begin{align}\label{cot formula}
\tang{\bth,\bmu}=F_{\bar{u}}(\tang{\bti,\bmu})\where F_{\bar{u}}(x)=\frac{1+\sigma\frac{\Lambda x}{\sqrt{1+x^2}} -2\nu}{\sigma\sqrt{\frac{\Lambda^2}{1+x^2}+\frac{1}{\bar{u}\rho}}},
\end{align}
We remark that \cite{castelli1996relative, lelarge2019asymptotic} studies mixture models and provides information theoretical bounds. Our bound complements these works by characterizing the performance of self-training which is a widely-used practical algorithm. We also characterize the benefit of using the acceptance threshold $\Gamma$ which is again a critical heuristic for the success of self-training. We suspect that one can analyze self-training performance for more general distributions and other base classifiers, instead of averaging estimator, by using tools from high-dimensional statistics and random matrix theory such as Gaussian min-max Theorem \cite{OymLAS,thrampoulidis2015regularized,stojnic2013framework} and approximate message passing \cite{donoho2009message,bayati2011dynamics}.
\vs
\subsection{Iterative self-training}\vs
Theorem \ref{sharp_bound_GMM} also allows us to analyze pseudo-labeling in an iterative fashion to show further improvement with more unlabeled data. Specifically, suppose we have $n$ labeled samples $\Sc=(\x_i)_{i=1}^n$ and $\tau\times u$ unlabeled samples $\Uc=(\x_i)_{i=n+1}^{n+\tau u}$. We first create the supervised model via \eqref{supervised model}. Then, we split $\Uc$ into $\tau$ disjoint sub-datasets $(\Uc_i)_{i=1}^\tau$. Starting from $\bt_0=\bti$ of \eqref{supervised model}, we iteratively apply self-training via pseudo-labeling \eqref{pseudo_label_est} to obtain
\begin{align}
\bt_{i}=\text{self-train}(\bt_{i-1},\Uc_i)\quad\text{for}\quad 1\leq i\leq \tau.\label{self train iter}
\end{align}
The final model is then equal to $\bth=\bt_\tau$. Note that the asymptotic co-tangent of self-training with $\tau$ iterations will be given by $F^\tau(x)$ where $x$ is the co-tangent of the initial supervised model. The following theorem establishes the asymptotic performance of this procedure.
\begin{theorem}[Iterative self-training bound] \label{lem iter bound}Set $\bar{n}=n/p$ and $\bar{u}=u/p$. Let $\Sc=(\x_i,y_i)_{i=1}^n$ and $\Uc=(\x_i)_{i=n+1}^{n+\tau u}$ be independent datasets with i.i.d.~samples generated according to Binary GMM. Obtain the model $\bth$ via applying $T$ iterations of the iterative self-training \eqref{self train iter} to the supervised model \eqref{supervised model}. Recall the co-tangent evolution formula of \eqref{cot formula}. We have that
\begin{align}
\lim_{p\rightarrow \infty}\tang{\bth,\bmu}\conv F_{\bar{u}}^\tau(\sqrt{\bar{n}}/\sigma).\label{our formula}
\end{align}
\end{theorem}

\begin{figure}[t!]

	\begin{subfigure}{2.2in}
		\includegraphics[scale=0.26]{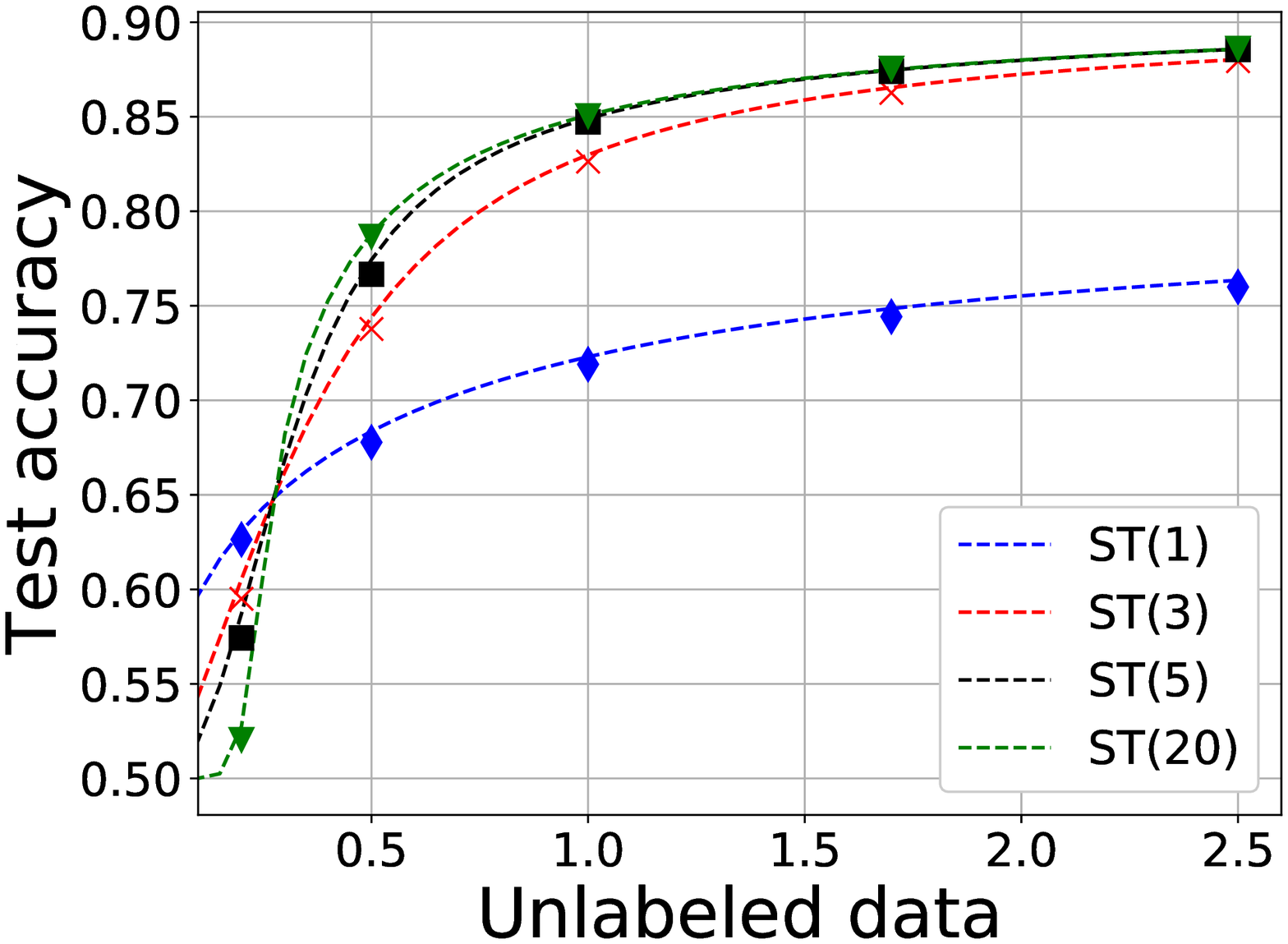}\caption{The impact of more self-training iterations on the model accuracy.}\label{fig1g}
	\end{subfigure}
	\begin{subfigure}{2.2in}
		\includegraphics[scale=0.26]{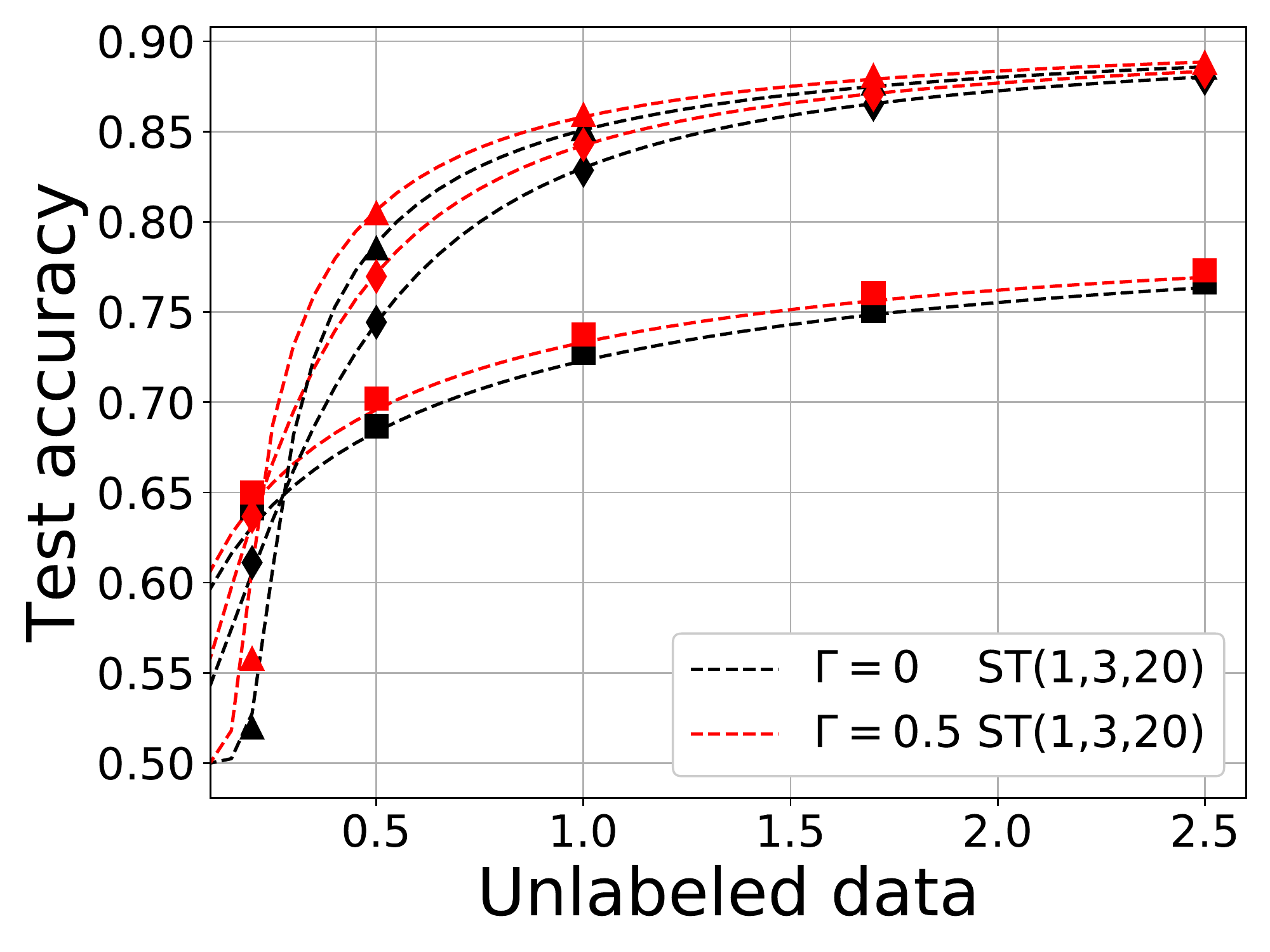}\caption{Comparing acceptance thresholds of $\Gamma=0$ vs $\Gamma=0.5$.}\label{fig2g}
	\end{subfigure}
	\begin{subfigure}{2.2in}\vspace{-3pt}
		\includegraphics[scale=0.28]{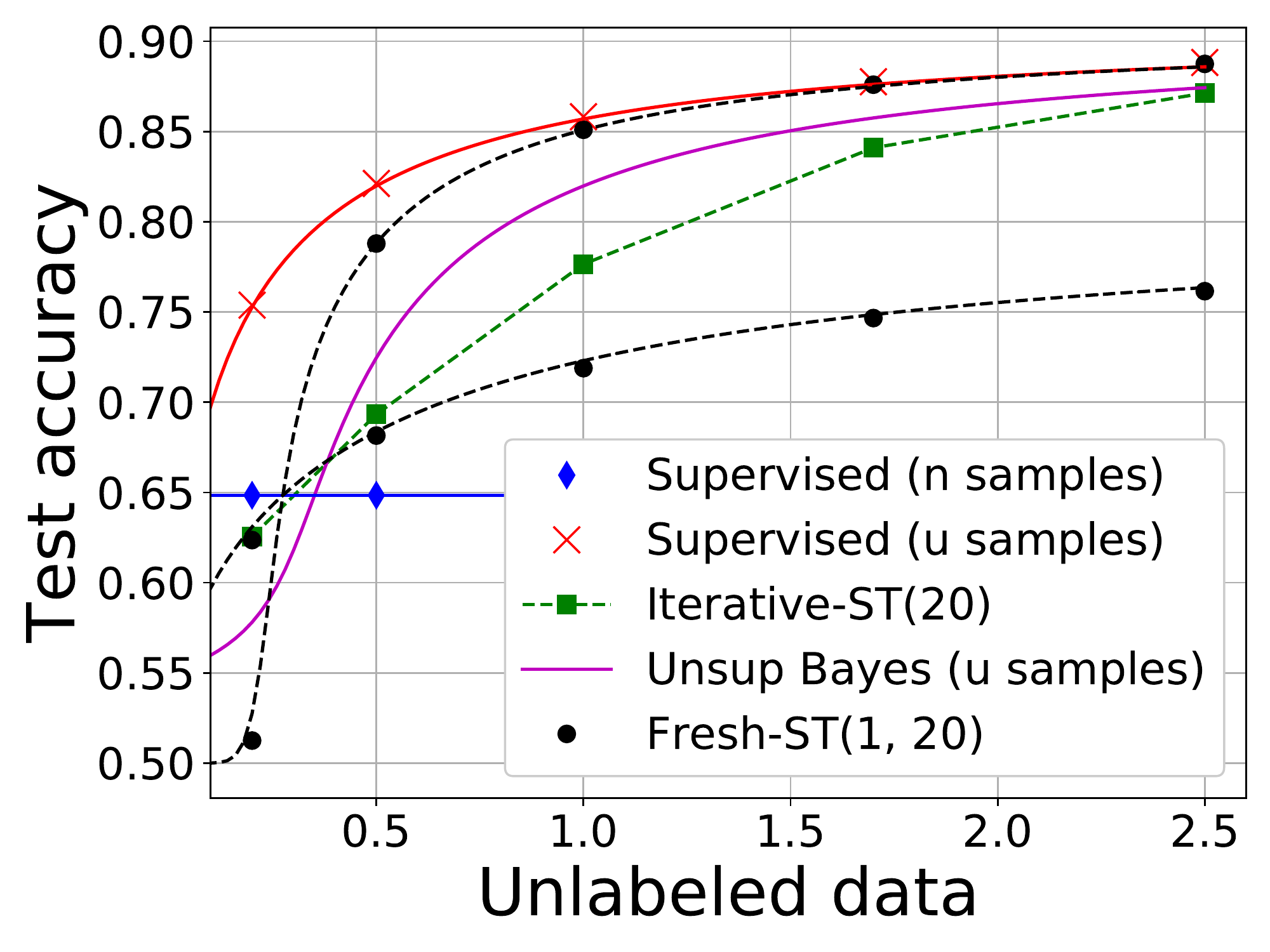}\vspace{-5pt}\caption{Comparison of different baselines at $\Gamma=0$.}\label{fig3g}
	\end{subfigure}
	\caption{{\small{$p=400$, $\bar{n}=n/p=0.05$, $\Gamma=0.5$, $\sigma=0.75$. $x$-axis is the unlabeled data amount $\bar{u}=u/p$. In Figures (a) and (b), ST($\tau$) refers to self-training repeated $\tau$ times with new batch of unlabeled data (same as Fresh-ST). Larger $\tau$ corresponds to the line with better accuracy. All lines are theoretical predictions except the Iterative-ST.}}}
	\label{fig:performance} \vspace{-0.2cm}
\end{figure}

%

Let us call this model Fresh-ST (ST for self-training) as each iteration requires fresh batch of unlabeled data. Figure~\ref{fig1g} and Figure~\ref{fig2g} illustrate the the test performance associated with this iterative approach. The parameters in these figures are as follows. We set labeled data amount to be $\bar{n}=0.05$ and unlabeled data amount $\bar{u}$ is varied along the $x$ axis. The noise level is $\sigma=0.75$ and the input dimension is $p=400$. The dashed lines are our formula \eqref{our formula}. We see from Figure~\ref{fig1g} that the test performance improves as the amount of unlabeled data increases (here $\Gamma=0$). The self-training iterations also improve the test accuracy as long as the unlabeled data amount is above the fixed point of the $F_{\bar{u}}$ function. In other words, we need $\bar{u}$ larger than a threshold $u_*$ where $u_*$ preserves the co-tangent of the initial supervised model i.e.~$F_{u_*}(\tang{\bti,\bmu})=\tang{\bti,\bmu}$. Clearly this threshold $u_*$ depends on the initial supervised model (i.e.~the amount of labeled training data) as well as the noise level $\sigma$. Figure~\ref{fig2g} demonstrates that choosing a proper acceptance threshold $\Gamma$ can improve the test performance over always choosing $\Gamma=0$. We observe that benefit of optimizing $\Gamma$ is more noticeable when there are fewer unlabeled data. Also optimizing $\Gamma$ can shift the fixed point of the $F_{\bar{u}}$ function so that less unlabaled data is required for improvement.\vs

Figure~\ref{fig3g} provides multiple baselines to compare our self-training bounds ($\Gamma=0$) (blue, red, green curves). The blue curve is the performance of the initial model which only uses $n$ labels. The red curve is the performance of a supervised model that uses $u$ labeled samples. Note that, this curve is not necessarily an upper bound on the
performance of the Fresh-ST however provides a natural reference. The magenta curve is the accuracy of the unsupervised Bayes optimal classifier using $u$ input samples. Finally, the green line is the iterative self-training where we always use the same unlabeled dataset with $u$ samples. Specifically, we apply the iterations $\bt_{i+1}=\text{self-train}(\bt_i,\Uc)$ for $1\leq i\leq \tau=20$. Let us call this Iterative-ST. We see that, repetitively applying self-training on the same dataset improves the performance over applying it only once (i.e.~green line is above the lower dashed black line). On the other hand, we also see the positive effect of using fresh unlabeled data on the test performance from Figure~\ref{fig3g}. Comparing the Fresh-ST with the empirical performance of Iterative-ST in Figure~\ref{fig3g} shows that the test performance substantially benefits from resampling. For instance, only 3 iterations of resampling can be noticeably better than many iterations of Iterative-ST. Intuitively, this is due to the fact that repeated self-training on the same dataset can guide the optimization to a suboptimal fixed point of the self-training iteration. This is also known as the confirmation bias of pseudo-labeling \cite{arazo2019pseudo}. In this example, fresh samples help get out of bad fixed points.

\noindent {\bf{Logistic regression:}} We next compare our averaging-based self-training \eqref{pseudo_label_est} to logistic regression. Given unlabeled data $\Uc$ and a linear classifier $\bti$, we first obtain the dataset $\Uc'$ of acceptable inputs and associated pseudo-labels by thresholding $\x^T\bti/\tn{\bti}$. We then solve logistic regression over $\Uc'$ to obtain a new linear classifier. The test performances of logistic-regression self-training are plotted in Figure~\ref{fig:logistic}. The labeled data fraction is $\bar{n}=0.2$ and the unlabeled data amount varies along x-axis, as in the case of Figure~\ref{fig:performance}. 
We set $\Gamma=0$ in Figure~\ref{fig1l}, and $\Gamma=1/2$ in 
Figure~\ref{fig2l}. For both Figure~\ref{fig1l} and 
Figure~\ref{fig2l},
the black dashed line refers to Fresh-ST iterations, and
green dashed line corresponds to self-training iterations with the same unlabeled data. Similarly, blue dashed line plots the test performance of supervised learning with $n$ samples and red dashed line plots the performance of supervised learning with $u$ samples for both figures.
The dashed lines in Figure \ref{fig1l} are logistic regression based algorithms whereas solid lines display the performance of the corresponding averaging estimator. Observe that averaging bounds are uniformly better which is not surprising given that the averaging estimator is Bayes optimal for GMM. We observe from Figure~\ref{fig1l} and Figure~\ref{fig2l} that 
the amount of unlabeled data has a positive effect on the test performance, and carrying out self-training iterations with fresh unlabeled data improves the performance. Comparing Figure~\ref{fig1l} with
Figure~\ref{fig2l}, we see how the acceptance threshold $\Gamma$ plays a critical role on the outcome. In fact, we find out from Figure~\ref{fig2l} that Fresh-ST can outperform supervised learning with $u$ samples, and regular iterative self-training can outperform regular supervised algorithm if the acceptance threshold $\Gamma$ is high enough.
The effect of $\Gamma$ on the test performance is also demonstrated by Figure~\ref{fig3l}, where we observe how picking an appropriate acceptance threshold boosts the test performance. We also see from Figure~\ref{fig3l} how the test performance gets better when the number of iterations increases. 

\begin{figure}[t!]

	\begin{subfigure}{2.2in}
		\includegraphics[scale=0.26]{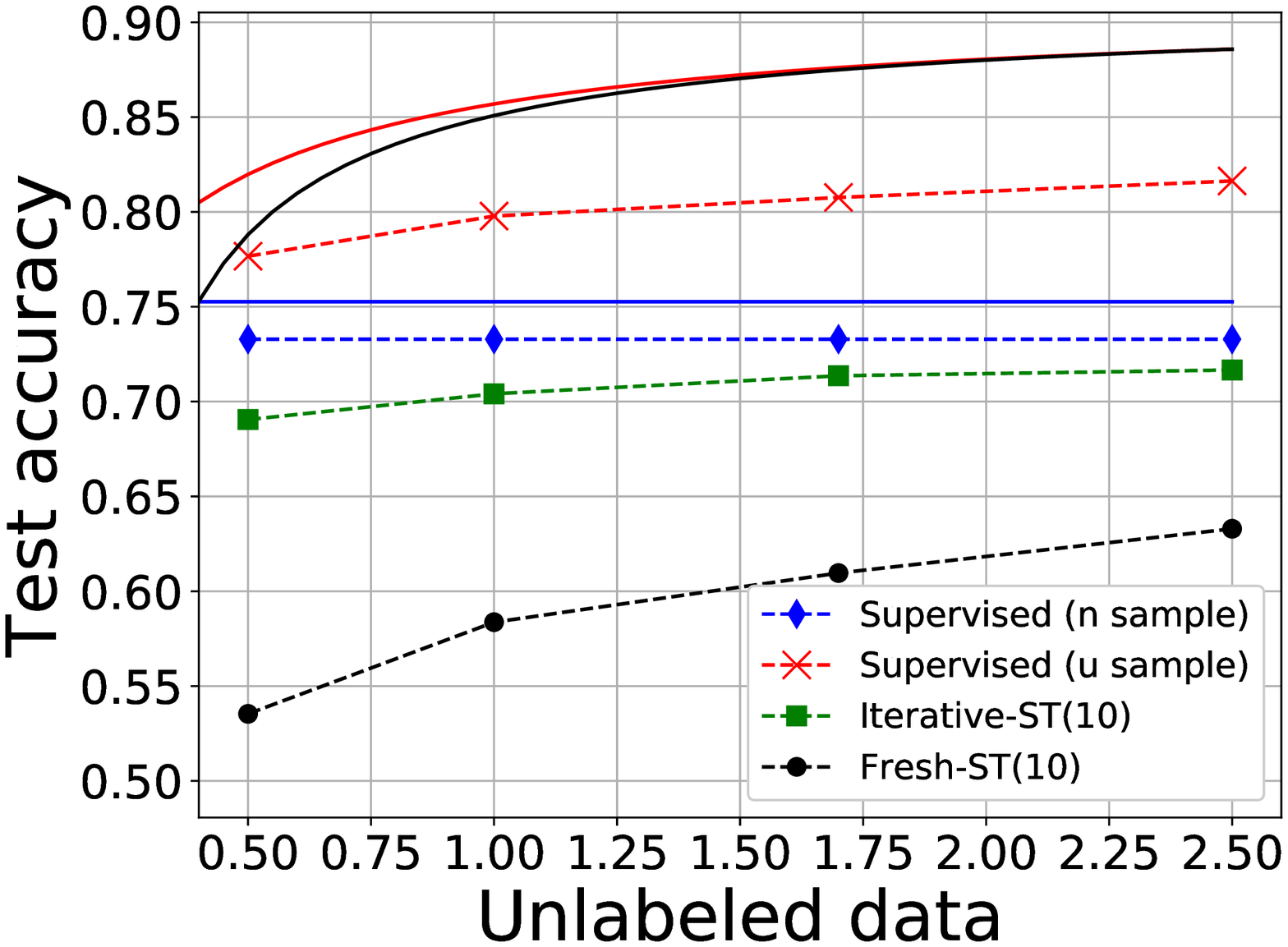}\caption{Comparing the performance of logistic regression and averaging estimator.}\label{fig1l}
	\end{subfigure}
	\begin{subfigure}{2.2in}
		\includegraphics[scale=0.26]{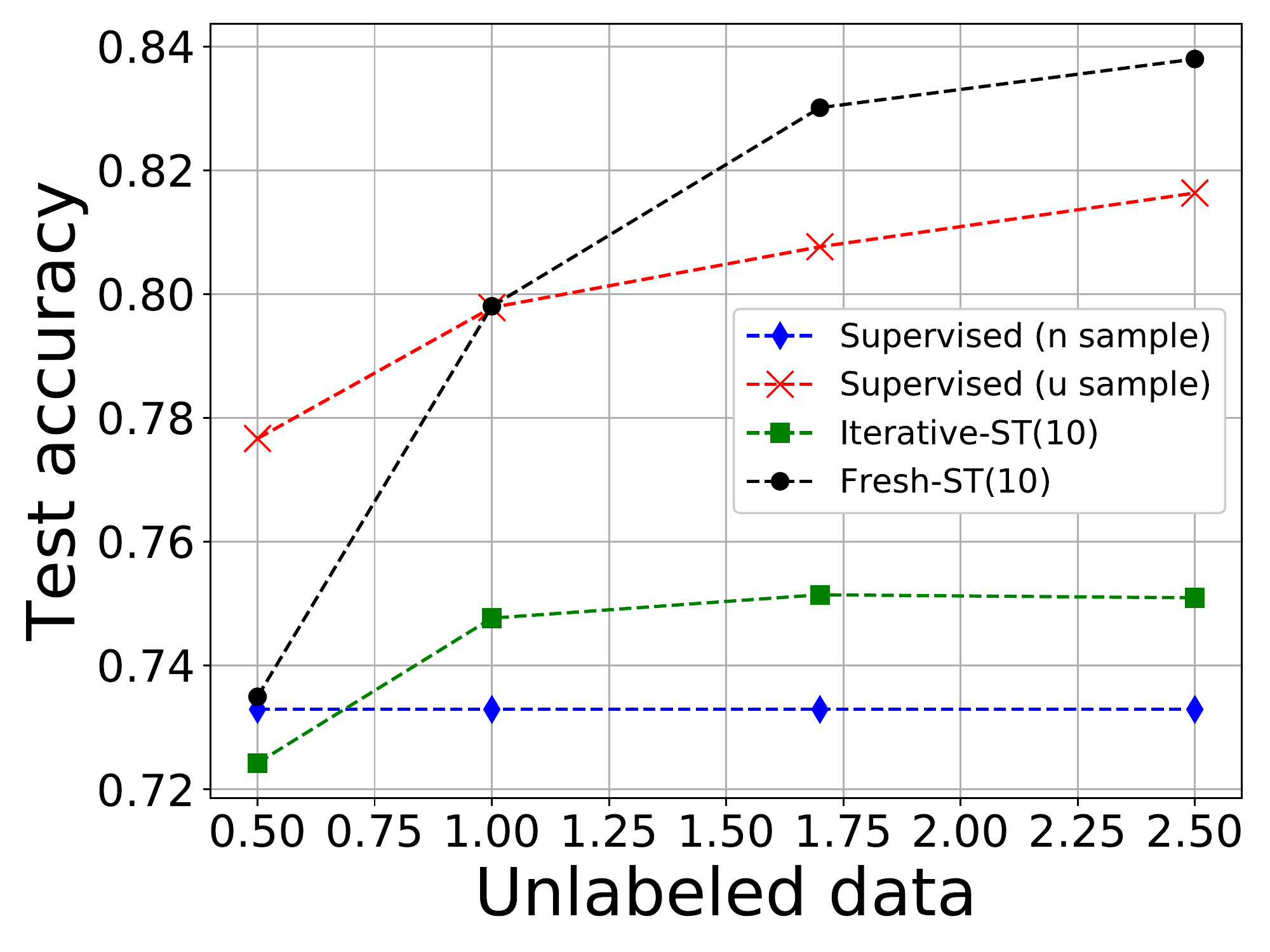}\caption{Performance of different approaches for logistic regression.}\label{fig2l}
	\end{subfigure}
	\begin{subfigure}{2.2in}
		\includegraphics[scale=0.26]{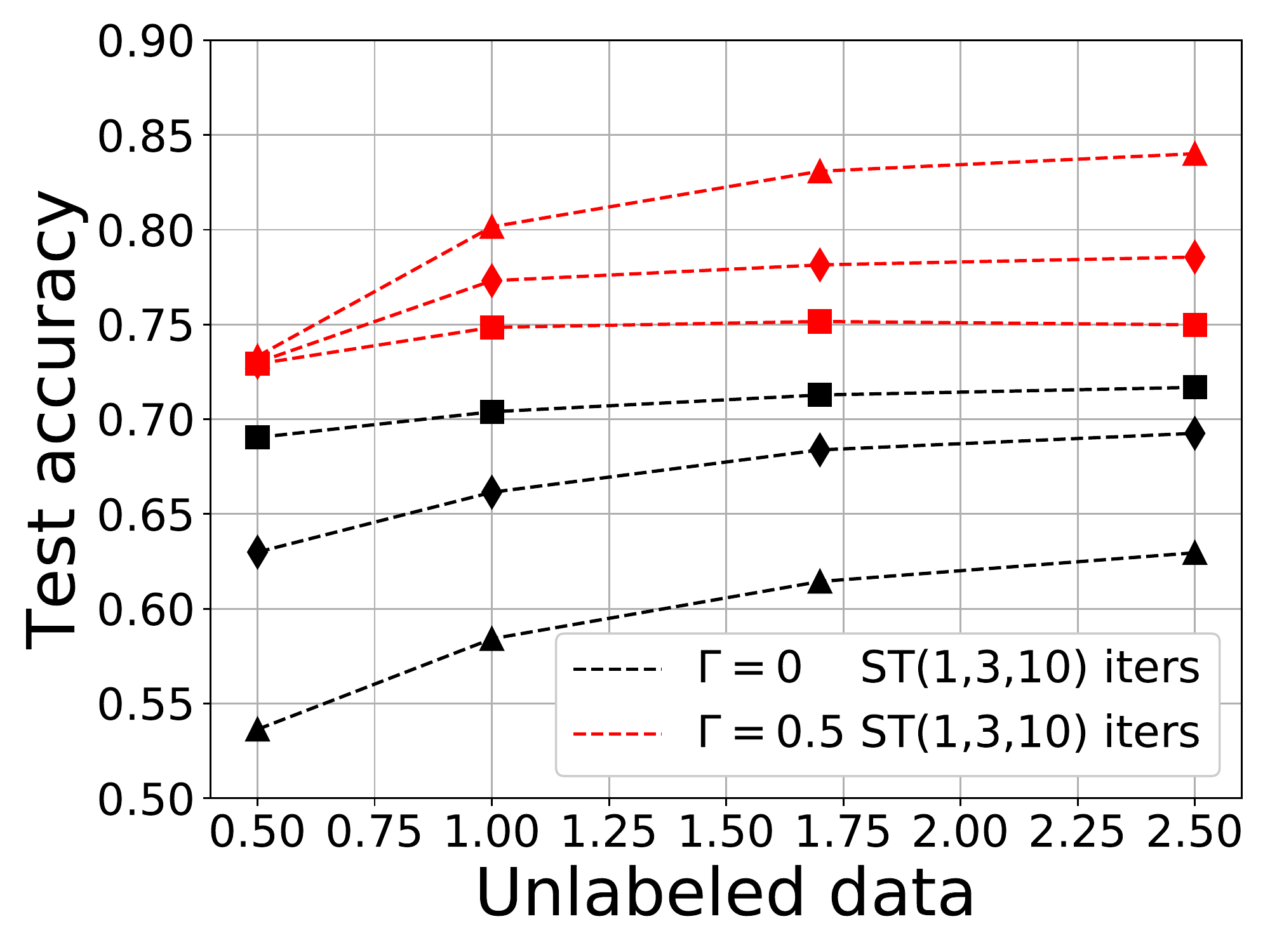}\caption{Logistic regression, comparison of $\Gamma=0$ vs $\Gamma=0.5$.}\label{fig3l}
	\end{subfigure}
	\caption{\small{Experiments on logistic regression: $\bar{n}=n/p=0.2$, $\sigma=0.75$, $p=400$. In (a), for the same color, solid lines are the performance of the averaging estimator and the dashed lines with markers correspond to the logistic regression. Fresh-ST (and ST in Fig.~(c)) uses fresh batch of unlabeled data at each self-training iteration. In Fig.~(c), (1,3,10) self-training iterations have markers $\Delta$, $\diamond$, $\square$ respectively.}}
	\label{fig4} \vspace{-0.2cm}
	\label{fig:logistic}
\end{figure}
\vs
\vs\section{Algorithmic Insights: Importance of Regularization and Margin}\vs\label{sec algo}

We consider here a particular binary mixture model involving a scalar random variable $X$, and investigate the conditions and learning setups under which the use of unlabeled data improves the alignment of the classifier with the ground-truth mixture mean $\bmu$ (and hence the accuracy). 
\begin{definition}[Generalized Mixture Model (Gen-MM)] The distribution $\Dc$ is given as follows. Fix a unit vector $\bmu\in\R^p$ and scalar $\sigma\geq 0$. Let $X,y,\g$ be independent random variables where $X$ is a scalar random variable with distribution $\Dc_X$, $\g\sim\Nn(0,\Iden_p)$, and $\Pro(y=1)=1-\Pro(y=-1)=1/2$. The input $\x$ is generated as \vs
\[
\x=yX\bmu+\sigma\g.
\]
\end{definition}
In this section, we provide algorithmic insights for the Gen-MM distribution which will shed light on the necessity of margin and importance of regularization. Here, our notion of margin is the gap between the class conditional distributions $X$ and $-X$. If $X$ is a positive random variable strictly bounded away from zero, then, we say there is a margin between the two classes since the distributions $X$ and $-X$ are away from each other. We first focus on a simplified scenario where we assume that we are provided an initial model $\bti$\footnote{Such an initial model can be obtained by minimizing the supervised risk $\Lch$ of \eqref{PL reg} or via \eqref{supervised model} as in Section \ref{sec gauss}.} and we use $\bti$ to label $\Uc$ and refine our estimate using pseudo-labeling. Focusing on least-squares loss and linear classifiers, in the infinite sample setup, this corresponds to the following problem\vs
\begin{align}
\bth=\frac{1}{2}\arg\min_{\bt}\E[1(|\bti^T\x|\geq \Gamma\tn{\bti})(\sgn{\bti^T\x}-\bt^T\x)^2].\label{PL sup}
\end{align}
Before investigating this problem, it is worth understanding the supervised loss. Setting $\beta=\bt^T\bmu$, the supervised quadratic loss is given by
\begin{align*}
\Lc_S(\bt)&=\E_{\Dc}[(y-\bt^T\x)^2]=\E_{X,\g}[(X\bt^T\bmu +\sigma \bt^T\g-1)^2]\\
&=\E[(X\beta-1)^2]+\sigma^2\tn{\bt^2}\\
&=\sigma_X^2\beta^2-2\mu_X\beta+1+\sigma^2\tn{\bt^2}
\end{align*}
The loss is minimized by choosing $\bt^\st=\beta^\st\bmu$ where $\beta^\st=\mu_X/(\sigma_X^2+\sigma^2)$. Additionally, this loss satisfies gradient dominance with respect to the global minima $\bt^\st$ (as it will be discussed further later on), thus gradient descent on population loss will quickly find $\bt^\st$. The question we are asking in this section is what happens when label information $y$ is replaced by the pseudo-labels $\sgn{\bti^T\x}$. Our first theorem picks $X$ to be the folded normal distribution (in words, $X$ is the absolute value of a standard normal variable) and shows a negative result on pseudo-labeling.\footnote{Folded normal has a nice simplifying nature during the theoretical analysis since $yX$ becomes standard normal.}
\vs\vs
\subsection{No Improvement with No Margin}\vs
\begin{theorem}\label{thm no} Pick $X$ to be the folded normal distribution (with density function $f_X(t)=\sqrt{2/\pi}\e^{-t^2/2}$) and any $\Gamma\geq 0$. Let $\bth$ be the solution of the population pseudo-labeling problem \eqref{PL sup}. For some scalar $c>0$ depending on $\sigma,\li\bmu,\bti\ri,\Gamma$, we have that $\bth=c\bti$.
\end{theorem}
The surprising conclusion from this theorem is that pseudo-labeling optimization \eqref{PL sup} do not lead to an improved model. $\bth$ remains parallel to the original model $\bti$ thus it will make the exact same label prediction as $\bti$. Observe that folded normal distribution has no margin since the distributions of $X$ and $-X$ both start from zero.

\vs\subsection{Improvement with Margin}\vs
In contrast to the result above, the following theorem shows that if there is a margin in the distribution of $X$ (i.e.~$X$ is strictly bounded away from zero), self-training does lead to an improved solution.
\begin{theorem}\label{thm margin} Fix $1\geq \gamma\geq \sigma >0$. Let $X$ satisfy the second moment condition $\E[X^2]=1$ and the margin condition $M\gamma\geq X\geq \gamma$. Let $\bth$ be the solution of the population self-training problem \eqref{PL sup}. For $\Gamma=0$, setting $\corr{\bti,\bmu}= \alpha$, we have that
\[
\tang{\bth,\bmu}\geq \frac{\sigma\e^{C}}{4}(\gamma(1- 6\e^{-C}M)).
\]
where $C=\frac{\alpha^2\gamma^2}{2\sigma^2}$. Specifically, if $\alpha \gamma>\sqrt{2\log (12 M)}\sigma$, we find $\tang{\bth,\bmu}\geq 0.1{\sigma\gamma\e^{\frac{\alpha^2\gamma^2}{\sigma^2}}}$.
\end{theorem}
Note that $\tang{\bth,\bmu}$ can be arbitrarily larger than the initial value $\tang{\bti,\bmu}$. As $\sigma$ decreases, $\tang{\bth,\bmu}$ increases exponentially fast in the margin $\gamma$ and the initial correlation $\alpha$ and $\bth$ becomes quickly aligned with the optimal direction $\bmu$. This should be contrasted with Theorem \ref{thm no} where $\bth$ remains aligned with the initial model $\bti$ which implies no improvement.

\vs\subsection{Benefits of Regularization}\vs
In this section, we show that with proper regularization, distributional bias of the data can push the solution towards the global minima (i.e.~a classifier perfectly aligned with $\bmu$). We consider two type of regularizations.
\begin{itemize}
\item {\bf{Ridge regression:}} Consider the ridge regularized version of \eqref{PL sup} given by
\begin{align}
\bth=\frac{1}{2}\arg\min_{\bt}\E[1(|\bti^T\x|\geq \Gamma\tn{\bti})(\sgn{\bti^T\x}-\bt^T\x)^2]+\la\tn{\bt}^2.\label{ridge sup}
\end{align}
\item {\bf{Early-stopping:}} Apply a single gradient iteration which corresponds to the averaging estimator of Section \ref{sec gauss}. This is given by the estimator
\begin{align}
\bth=\E[1(|\bti^T\x|\geq \Gamma\tn{\bti})\cdot\sgn{\bti^T\x}\cdot\x].\label{early sup}
\end{align}
\end{itemize}
In both cases, we show that regularization has a power-iteration-like affect which emphasizes the distributional bias of the data and picks up the central direction $\bmu$. Our first result characterizes the performance of the ridge regularization.
\vs\begin{lemma}[Ridge regression] \label{lem ridge}Set $\Gamma=0$ and let $X$ have folded normal distribution. Define the strictly increasing function
\[
\kappa(\la)=\frac{1+\sigma^2}{\sigma^2}\frac{\sigma^2+\la}{1+\sigma^2+\la}.
\]
Suppose $\bth$ is the solution of \eqref{ridge sup}. We have that $\tang{\bth,\bmu}=\kappa(\la)\tang{\bti,\bmu}$.
\end{lemma}
Observe that $\kappa(\la)>1$ and unlabeled data leads to provable improvement for any positive regularization parameter $\la>0$. Our second result characterizes the performance of early-stopping (i.e.~single iteration).


\begin{lemma}[Early-stopping] \label{lem early}Suppose $\bti^T\bmu=\alpha$ and let $X$ have folded normal distribution. Suppose  $\bth$ is the solution of \eqref{early sup}. We have that
\begin{align}
\tang{\bth,\bmu}=(1+\sigma^{-2})\tang{\bti,\bmu}\label{early rat}.
\end{align}
\end{lemma}
Here, observe that the improvement in the co-tangent $\tang{\bth,\bmu}$ is captured by the signal-to-noise ratio. Since $X$ is folded normal, the covariance matrix of the data obeys
\[
\E[\x\x^T]=\sigma^2\Iden+\bmu\bmu^T.
\]
The eigenvalue along the signal direction $\bmu$ is $1+\sigma^2$ whereas the orthogonal eigenvalues along the noisy directions are $\sigma^2$ and the ratio between them is $(1+\sigma^2)/\sigma^2=1+\sigma^{-2}$.

\vs\vs\subsection{Importance of Regularization in Logistic Regression}\label{sec degenerate}\vs
Note that regularization is also critical for ensuring the success of self-training when it comes to classification loss functions as well. Examples include logistic loss, hinge loss and exponential loss. All of these loss functions have the common form $\ell(y,\hat{y})=\ell(y\hat{y})$, are monotonically decreasing \cite{gunasekar2018characterizing}, and satisfy the limit $\lim_{t\rightarrow \infty} \ell(t)=0$. For instance hinge loss is given by $\ell(y,\hat{y})=(1-y\hat{y})_+$ and exponential loss is given by $\ell(y,\hat{y})=\e^{-y\hat{y}}$. For logistic and exponential loss, the training loss can never achieve zero and the model parameters have to indefinitely grow to minimize the training loss. The pseudo-label loss function is obtained by setting $y=\sgn{\hat{y}}$ so that the unlabaled example has loss equal to $\ell(|\hat{y}|)$.

In this section, we will briefly argue that, regularization is critical for enabling self-training/pseudo-labeling to find non-trivial models. The basic intuition is that, without regularization, the self-training loss can easily achieve zero while preserving the label decision of the original classifier. In other words, there is a trivial global optimum. For instance, suppose we scale the final (i.e.~logit) layer of a neural network by $\alpha$. Then, this network will output the logits $\alpha \hat{y}$ rather than $\hat{y}$. For $\alpha>0$, the class decision for $\alpha\hat{y}$ is exactly same as $\hat{y}$. However for $\alpha \geq 1$, the training loss decreases from $\ell(|\hat{y}|)$ to $\ell(\alpha|\hat{y}|)$. In general, as long as $\hat{y}\neq 0$, indefinitely enlarging $\alpha$ will asymptotically make the training loss zero. The following lemma provides a rigorous statement of this basic observation for a general function classes.
\begin{lemma} \label{lem simple stuff}Fix a prediction function $f:\R^p\rightarrow \R$. Consider the function class $\Fc=\{\alpha f \bgl \alpha \geq 0\}$. Suppose the loss function $\ell$ obeys $\lim_{t\rightarrow \infty} \ell(t)=0$ and the input distribution $\x\sim \Dc$ satisfies $\Pro_{\Dc}(f(\x)\neq 0)=1$. Define the population self-training loss $\Lct(f)=\E_{\Dc}[\ell(|f(\x)|)]$. We have that
\[
\lim_{\alpha\rightarrow\infty} \Lct(\alpha f)=0.
\]
\end{lemma}

Note that, the nonzero condition $\Pro_{\Dc}(f(\x)\neq 0)=1$ helps us push the loss to zero by increasing the scale $\alpha$. While this is a reasonably mild condition when the data has continuous distribution, we can also avoid this by considering an infinitesimal perturbation of $f$ to reach a similar conclusion (e.g.~using $\tilde{f}(\x)=f(\x)+g$ where $g$ is Gaussian noise with arbitrarily small variance).  

Similar to least-squares, regularization techniques such as ridge-regression and early-stopping can guide self-training towards useful models by preventing degenerate solutions (which requires $\alpha\rightarrow \infty$) provided in Lemma \ref{lem simple stuff}.

\begin{figure}[t!]

	\begin{subfigure}{3.3in}\centering
		\includegraphics[scale=0.35]{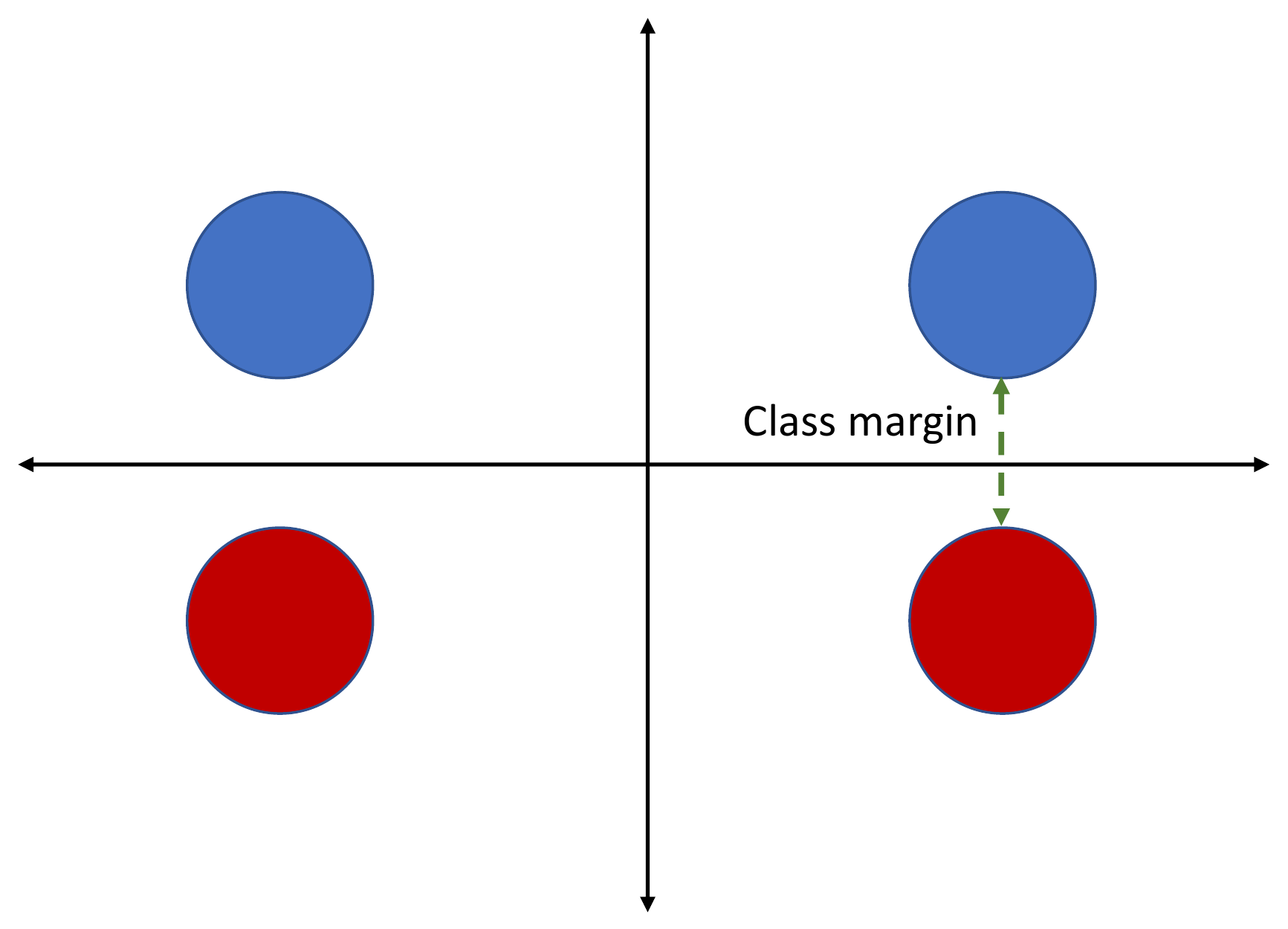}\caption{Example distribution for the labeled data. x-axis separates the classes.}\label{fig1cl}
	\end{subfigure}~~
	\begin{subfigure}{3.3in}\centering
		\includegraphics[scale=0.35]{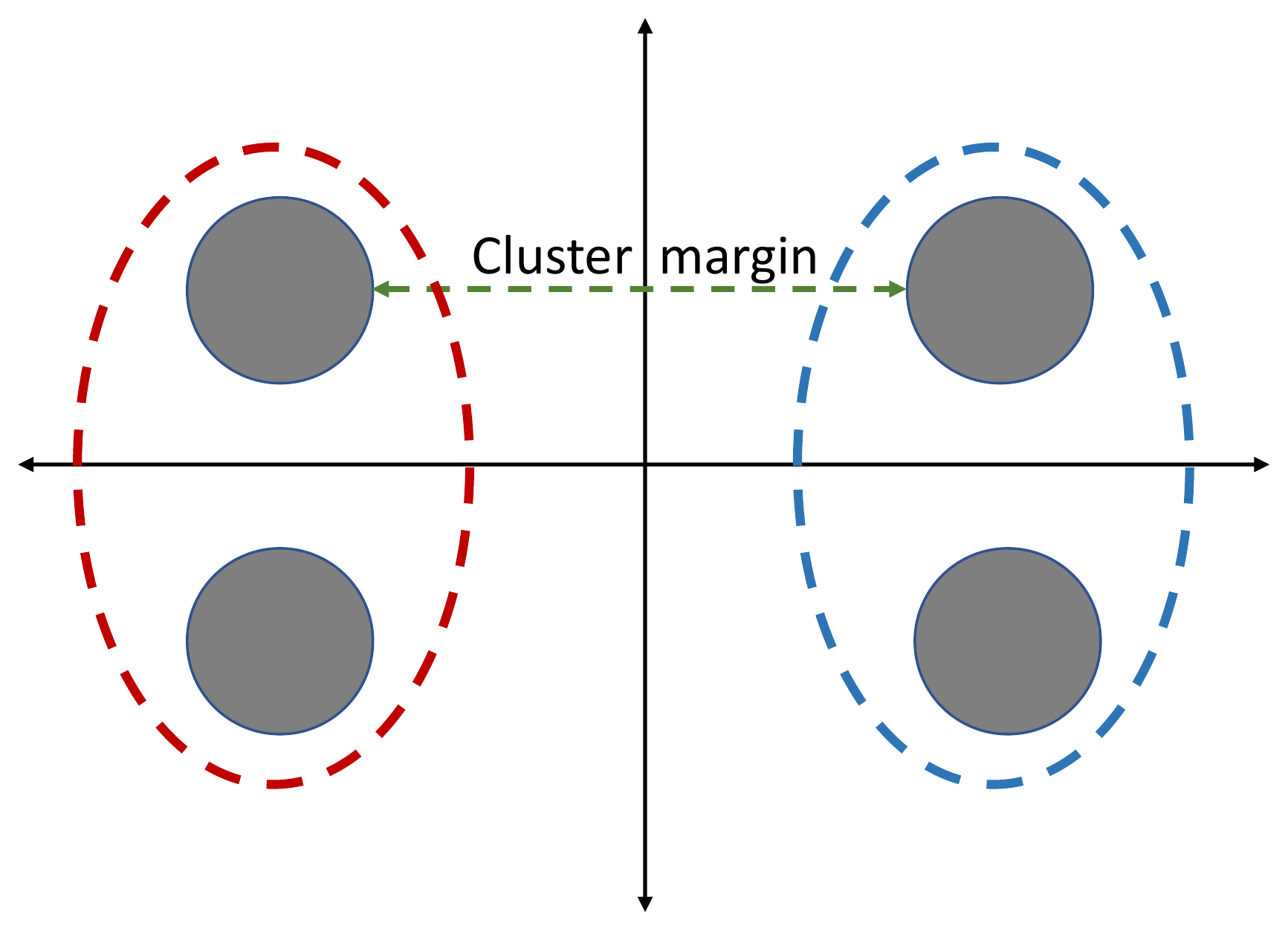}\caption{Unsupervised clustering induced by pseudo-labels. While x-axis separates the classes, y-axis maximizes the margin.}\label{fig2cl}
	\end{subfigure}
	\caption{Supervised training would find a model maximizing the class margin. Unsupervised training would find a model maximizing the cluster margin.}
	\label{fig:cluster} \vspace{-0.2cm}
\end{figure}

\section{Statistical Insights Beyond Mixture Models}\label{sec hetero}
In this section, we focus on statistical aspects of empirical risk minimization with self-training for general datasets. First, we show how a purely unsupervised notion of generalization based on self-training based clustering can be formalized based on cluster margin. Then, we connect self-training based semi-supervised learning to the more general problem of learning with heterogenous datasets

\subsection{Unsupervised Learning with Self-Training}

Classical statistical learning bounds such as Rademacher complexity arguments focus on labeled datasets and aims to show that the minimizer of empirical risk can accurately predict test labels. A natural question is how to assess the success of pseudo-label optimization without any labels. While it is not possible for self-training optimization to find the true class distributions \cite{balcan2010discriminative} without supervision, it is possible to argue that self-training loss can induce good clusters. {This is illustrated via an example distribution in Figure \ref{fig:cluster}. Figure \ref{fig1cl} shows the distribution of the labeled data where classes are separated along x-axis. While there is a margin between the classes, clustering along y-axis leads to a larger margin. Thus, the class distributions are not the ideal clusters and self-training will not be able to find the correct class assignments without supervision. However, minimizing a proper self-training loss should be able to find the margin-maximizing clustering (i.e.~ Figure \ref{fig2cl}).} Following this example, let us assess the clustering quality via margin, i.e.~ensuring that the input samples are away from the decision boundary. For instance, we can declare error if an input sample has margin less than $\gamma$. Such a clustering error can be defined as
\begin{align}
{\cal{E}}_\gamma(f)=\Pro_{\x\sim\Dc}(|f(\x)|\leq \gamma).\label{cluster error}
\end{align}
where $|f(\x)|$ is the margin with respect to the model's own decision. Here, recall that the absolute value $|f(\x)|$ naturally arises from the use of pseudo-labels. As discussed in Section \ref{sec degenerate}, common loss functions such as quadratic, logistic and hinge loss has the simplifying multiplicative form $\ell(y,\hat{y})=\ell(\hat{y}y)$ where $\hat{y}=f(\x)$ is the prediction and $y$ is the label. Plugging pseudo-label $\sgn{f(\x)}$ instead of the true label $y$ leads to $\ell(|f(\x)|)$. To encourage $\gamma$ margin smoothly, let us define the margin loss $\ell_\gamma(\cdot):[0,\infty)\rightarrow [0,1]$
\[
\ell_\gamma(x)=\begin{cases}0\quad\quad\quad\text{if}\quad x\geq 2\gamma\\-\frac{x}{\gamma}+2\quad~\text{if}\quad \gamma\leq x\leq 2\gamma\\1\quad\quad\quad\text{if}\quad 0\leq x\leq \gamma\end{cases}.
\]
Observe that $\ell_\gamma(x)$ is upper and lower bounded by the indicator functions as follows 
\[
1(x\leq \gamma)\leq \ell_\gamma(x)\leq 1(x\leq 2\gamma).
\] To proceed, given unlabeled samples $\Uc=(\x_i)_{i=1}^u$ and a function class $\Fc$, we show that solving the unsupervised empirical risk minimization problem 
\begin{align}
\hat f=\arg\min_{f\in \Fc}\frac{1}{\nt}\sum_{i=1}^{\nt}\ell_\gamma(|f(\x_i)|).\label{unsup}
\end{align}
can return a solution with good generalizability. Let $(\eps_i)_{i=1}^n$ be i.i.d.~Rademacher variables. Define the Rademacher complexity of $\Fc$ with respect to $\Uc$ to be
\[
\Rc_u(\Fc)=\frac{1}{n}\E[\sup_{f\in\Fc}\sum_{i=1}^u \eps_if(\x_i)].
\]

To make the dependence on the distribution $\Dc$ explicit, we will also use the notation $\Rc^{\Dc}_u(\Fc)$ later on. The following lemma follows from standard Rademacher complexity arguments\footnote{We are not aware of prior literature explicitly stating such a result however the proof does not require significant novelty over the standard Rademacher complexity arguments.} to show that $\hat f$ can induce a good clustering over the distribution $\Dc$ in terms of prediction margin.
\begin{lemma}[Self-Training Based Clustering] \label{unsup gen lemma}Sample unlabeled data $\Uc=(\x_i)_{i=1}^u\distas \Dc$. With probability at least $1-\delta$ over $\Uc$, the clustering risk \eqref{cluster error} of the solution $\hat f$ of the ERM \eqref{unsup} obeys
\[
{\cal{E}}_\gamma(\hat{f})\leq  \min_{f\in\Fc}{\cal{E}}_{2\gamma}(f)+\frac{2}{\gamma}\Rc_u(\Fc)+2\sqrt{\frac{\log(2/\delta)}{u}}.
\]
\end{lemma}
In words, this bound states that the $\gamma$-clustering error induced by the empirical minimizer $\hat f$ is upper bounded by the optimal $2\gamma$-clustering error plus the Rademacher complexity term. It is also important to note that this bound is scale invariant. If the functions in the hypothesis set $\Fc$ are scaled by a constant, the margin $\gamma$ can be scaled by the same constant and the bound would remain perfectly intact. Thus, the bound is essentially in terms of normalized margin i.e.~the margin normalized by the norm/magnitude of the functions. Recall that, if we fix $\gamma$ and scale up the functions $f$, it is trivial to obtain $0$ clustering error as discussed in Section \ref{sec degenerate}. However, this would not learn a meaningful clustering of the data.



\subsection{Learning with Weak Supervision with Relation to Self-Training}

We discussed how pseudo-labels can help finding generalizable clusterings of the inputs however it is not clear how they can help towards identifying correct classes. To this aim, in this section, we discuss the fundamental principles of learning with heterogeneous distributions where the primary motivation is jointly learning from labeled and unlabeled datasets. Let $\Sc=(\z_i)_{i=1}^n\distas\Dc$ and $\Sct=(\zt_i)_{i=1}^u\distas\tilde{\Dc}$ be i.i.d. datasets with possibly different distributions. In this section, we consider the setup where $\Dc$ is the primary distribution of interest and $\tilde{\Dc}$ provides side information about $\Dc$. Specifically, our goal is finding a model achieving small population risk over $\Dc$. Given a loss function $\ell$ and function class $\Fc$, we wish to find $f\in\Fc$ achieving small population risk
\begin{align}
\Lc(f)=\E_{\z\sim \Dc}[\ell(f(\z))].
\end{align}
With this point of view, the dataset $\Sc$ provides strong supervision and $\Sct$ provides weak-supervision as it has a different distribution. In case of semisupervised learning, $\Sc$ contains labeled data $(\z_i)_{i=1}^n=(\x_i,y_i)_{i=1}^n$ and $\Sct$ contains unlabaled data $(\zt_i)_{i=1}^u=(\x_i)_{i=n+1}^{n+u}$. Of particular interest, we focus on the scenario where weak-supervision dataset is larger than strong supervision i.e. $u\gg n$.

\noindent {\bf{Numerical intuitions on heterogeneous losses:}} To proceed, we would like to formulate a problem which jointly uses $\Sc$ and $\Uc$. We first start with some numerical intuition towards this goal with a focus on GMM distribution of Def.~\ref{GMM def} with variance $\sigma^2=1$. Let us use linear classifier and quadratic loss. Then, for $(\x,y)$ distributed as GMM, define the supervised and unsupervised population losses, which corresponds to strong and weak supervision respectively, as follows 
\begin{align}
\text{Supervised:}~&\Lc(\bt)=\frac{1}{2}\E[(y-\x^T\bt)^2]\\
\text{Unsupervised:}~&\Lct(\bt)=\frac{1}{2}\frac{\E[1(|\x^T\bt|\geq \Gamma\tn{\bt})(\sgn{\x^T\bt}-\x^T\bt)^2]}{\Pro(|\x^T\bt|\geq \Gamma\tn{\bt})},\label{unsup loss}
\end{align}
where $\Gamma$ is the acceptance threshold for self-training. In Figure \ref{fig:loss}, we plot these supervised and unsupervised population losses for parameters $\bt$ along the $\bmu$ direction where $\pm \bmu$ are the mixture centers and $\Gamma=0$. We choose $\bt=\alpha\bmu$ where $\alpha$ is the scaling parameter (x-axis) and $y-$axis shows the loss associated with $\bt$. In Figure \ref{fig1}, the supervised loss curve is shown in blue which is convex and have a unique global minimum around $\alpha=0.5$. The red curve shows the unsupervised loss $\Lct(\bt)$ (purely self-training/pseudo-labels). The unsupervised loss has two global minima (symmetrically located) and one of these minima are closely located to the global minimum of the supervised loss. Also observe from Figure \ref{fig1} that the unsupervised loss is always less than the supervised loss over the entire $\alpha$ range as the pseudo-label
$\text{sgn}(f(\zt))$ induced by the data is guaranteed to result in a smaller or equal loss compared to that of the actual label. For semisupervised learning, we consider two types of regularization
\begin{align}
{\bf{Regularized:}}&\quad\Lc_{\text{semi}}(\bt)=(1-\rho)\Lc(\bt)+\rho\Lct(\bt)\label{reg form}\\
{\bf{Constrained:}}&\quad\Lc_{\text{semi}}(\bt)=\Lc(\bt)\quad\text{subject to}\quad \Lct(\bt)\leq \Xi.\label{const form}
\end{align}

\begin{figure}[t!]
	\begin{subfigure}{2.2in}
		\includegraphics[scale=0.27]{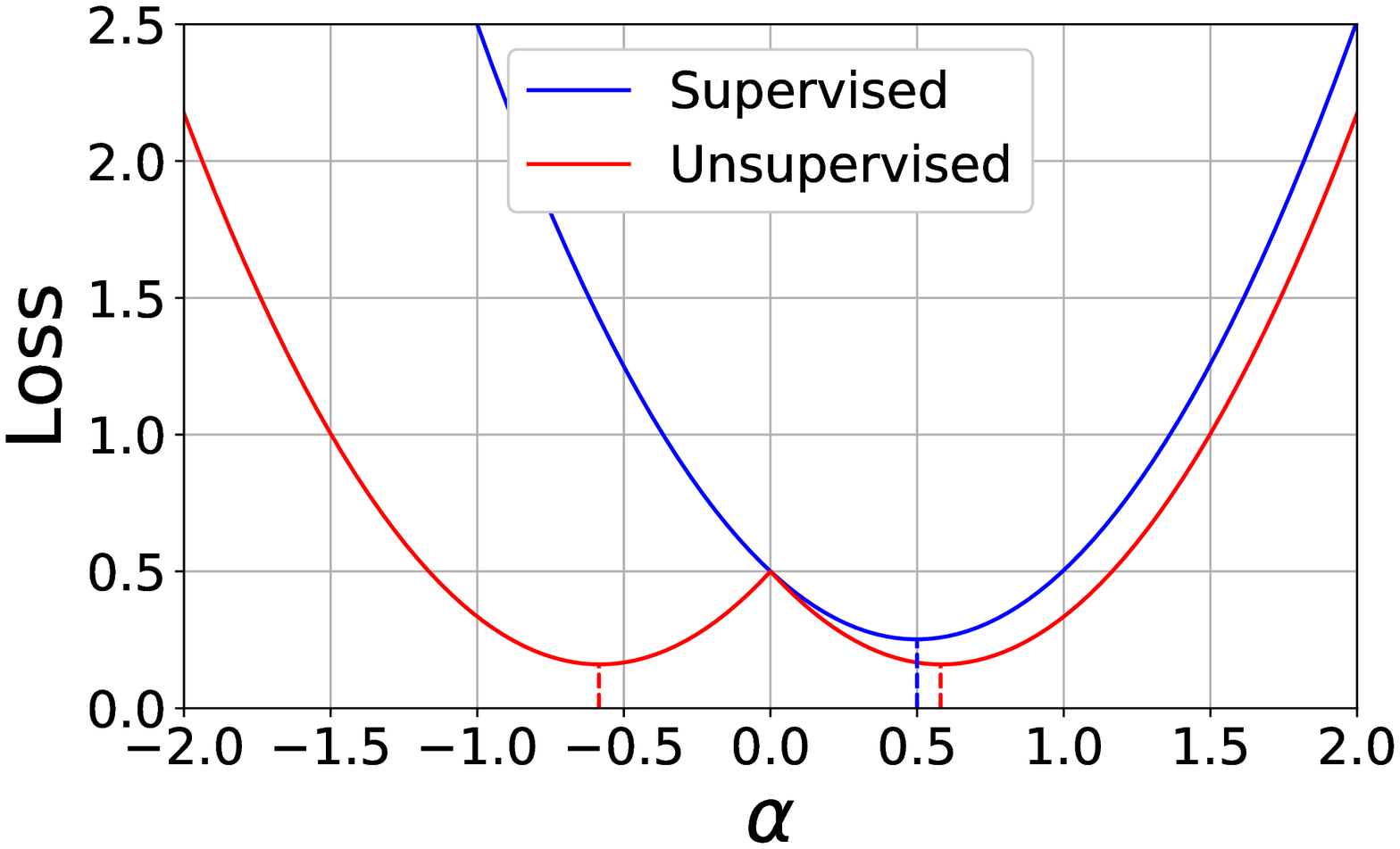}\caption{Landscapes of unsupervised (self-training) and supervised loss. $\Gamma=0$.}\label{fig1}
	\end{subfigure}~~~
	\begin{subfigure}{2.2in}
		\includegraphics[scale=0.27]{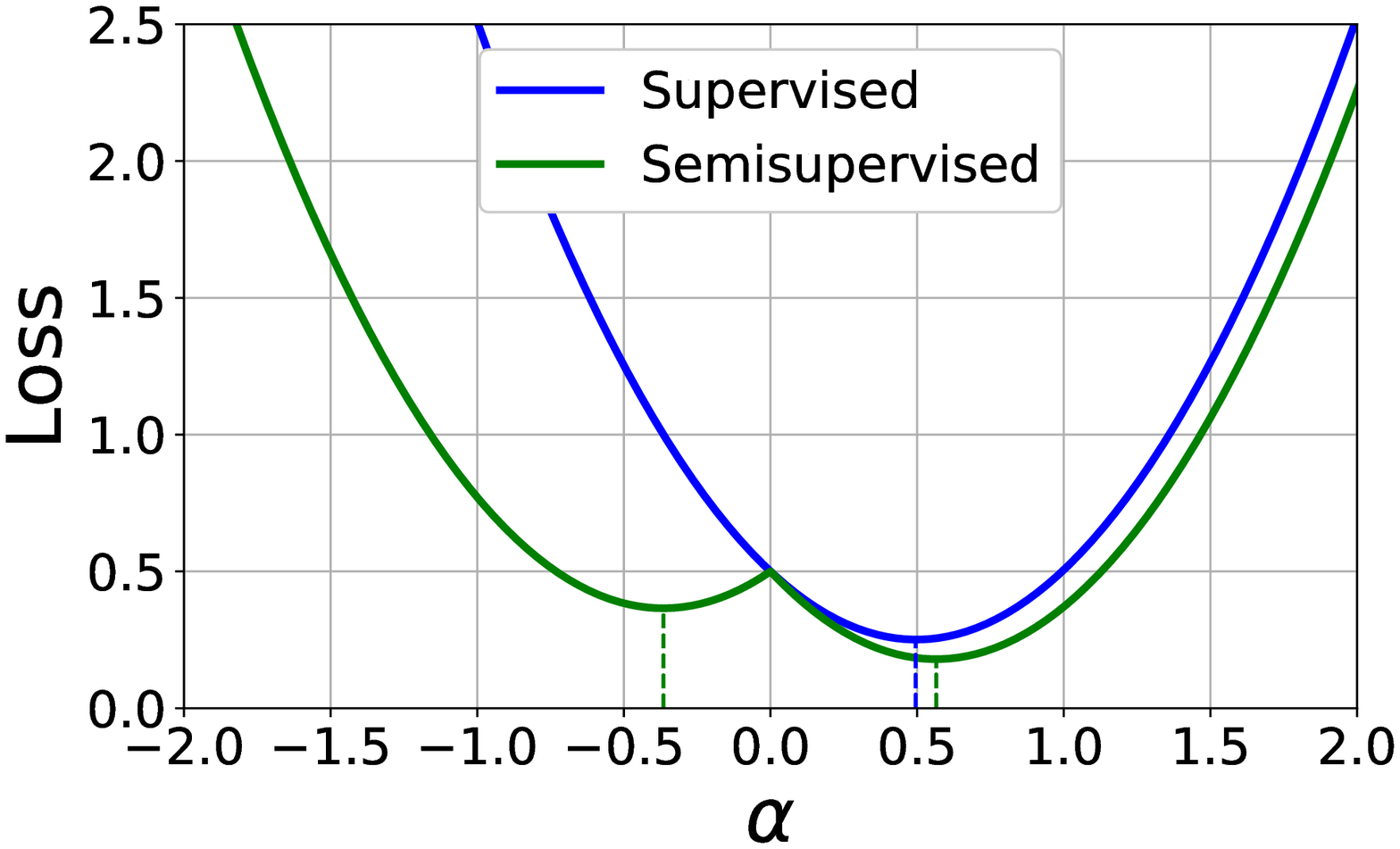}\caption{Semisupervised learning with $\rho=80\%$ self-training regularization}\label{fig2}
	\end{subfigure}~~~
	\begin{subfigure}{2.2in}
		\includegraphics[scale=0.27]{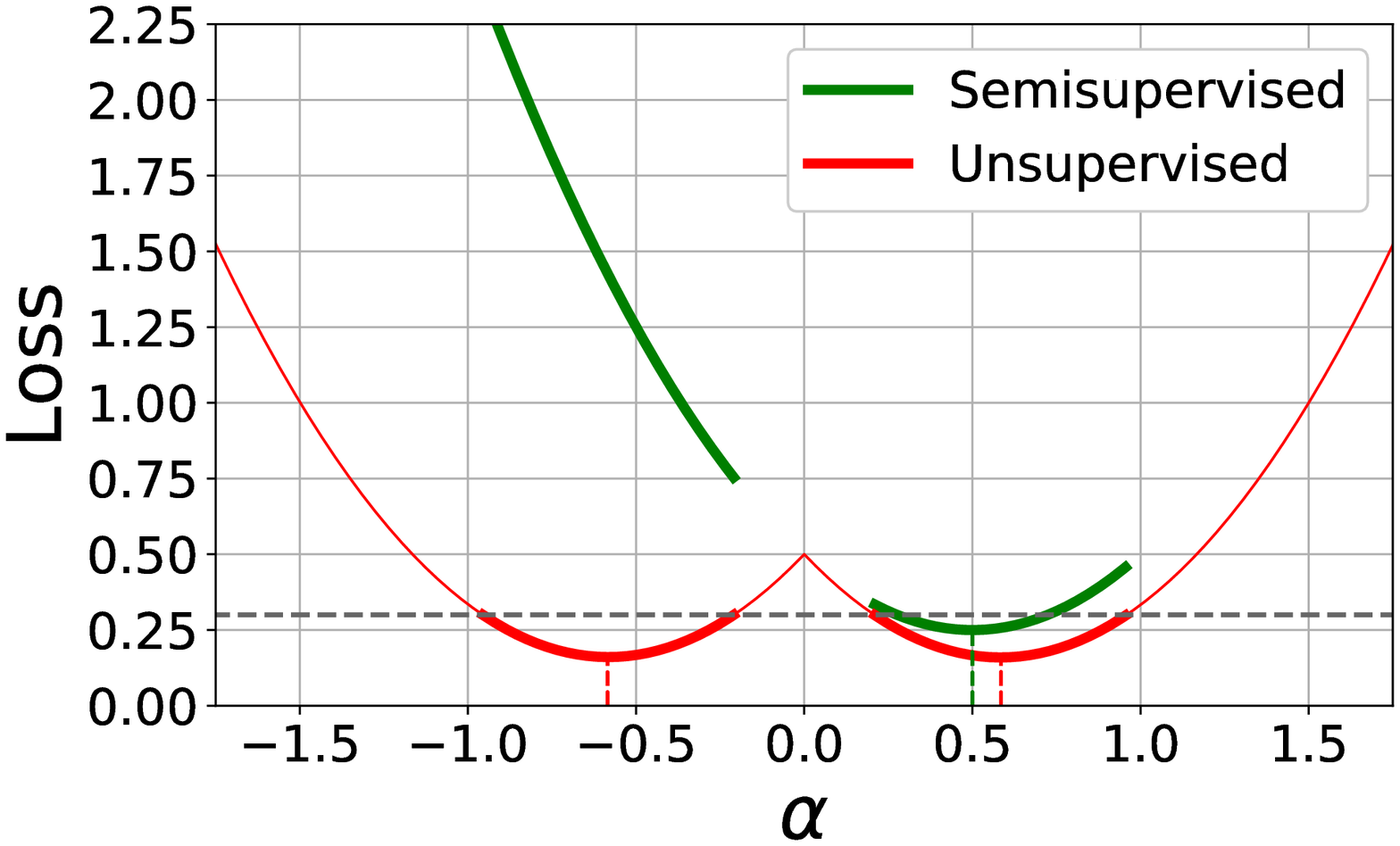}\caption{Semisupervised learning with self-training constraint $\Xi=0.3$.}\label{fig3}
	\end{subfigure}
	\caption{Loss landscape of self-training based semi-supervision for different regularization.}
	\label{fig:loss} \vspace{-0.2cm}
\end{figure}

\begin{figure}[t!]
	\begin{subfigure}{2.2in}
		\includegraphics[scale=0.27]{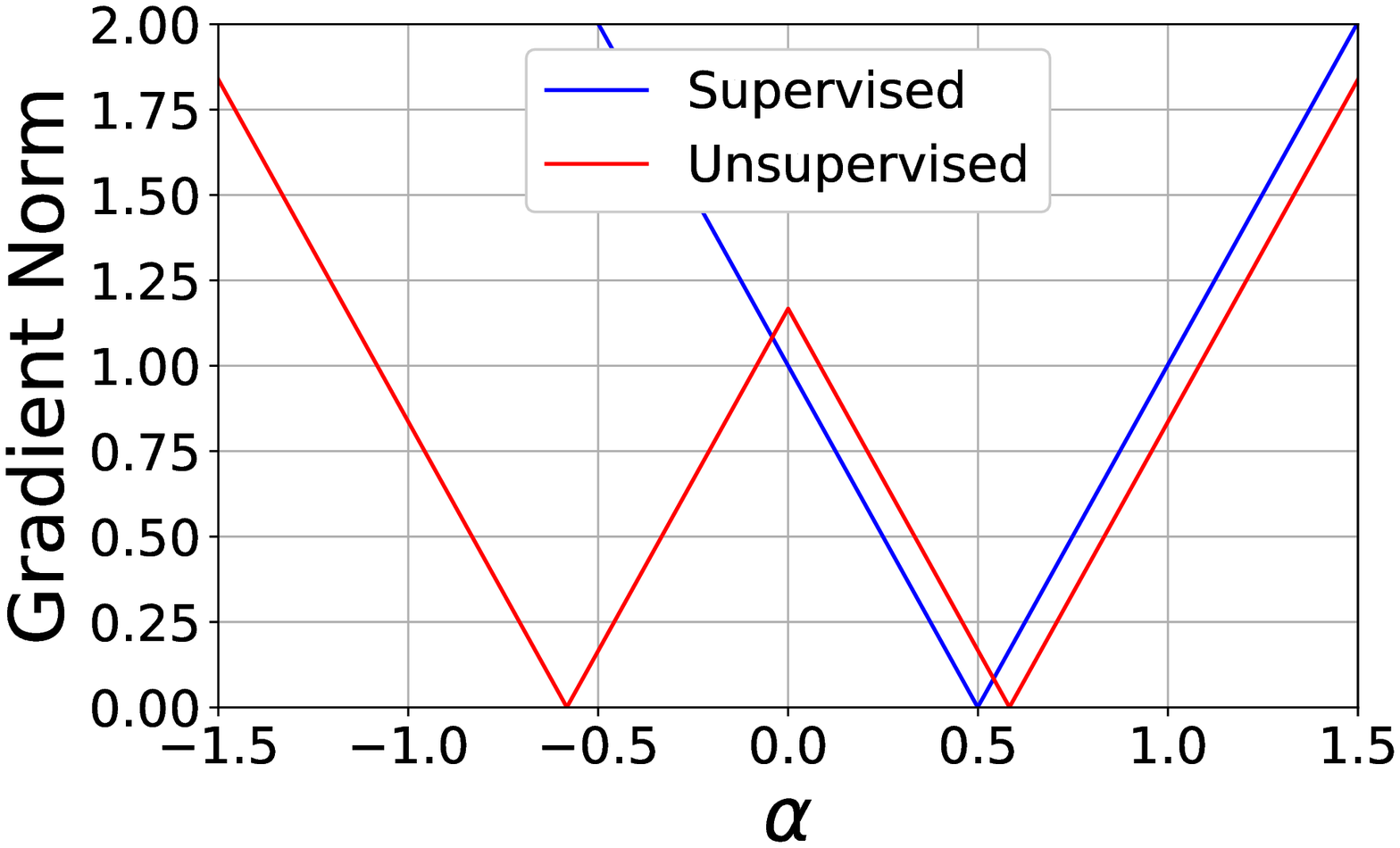}\caption{The gradient length (in $\ell_2$ norm) associated to the losses in Figure \ref{fig1}.}\label{fig_1g}
	\end{subfigure}~~~
	\begin{subfigure}{2.2in}
		\includegraphics[scale=0.27]{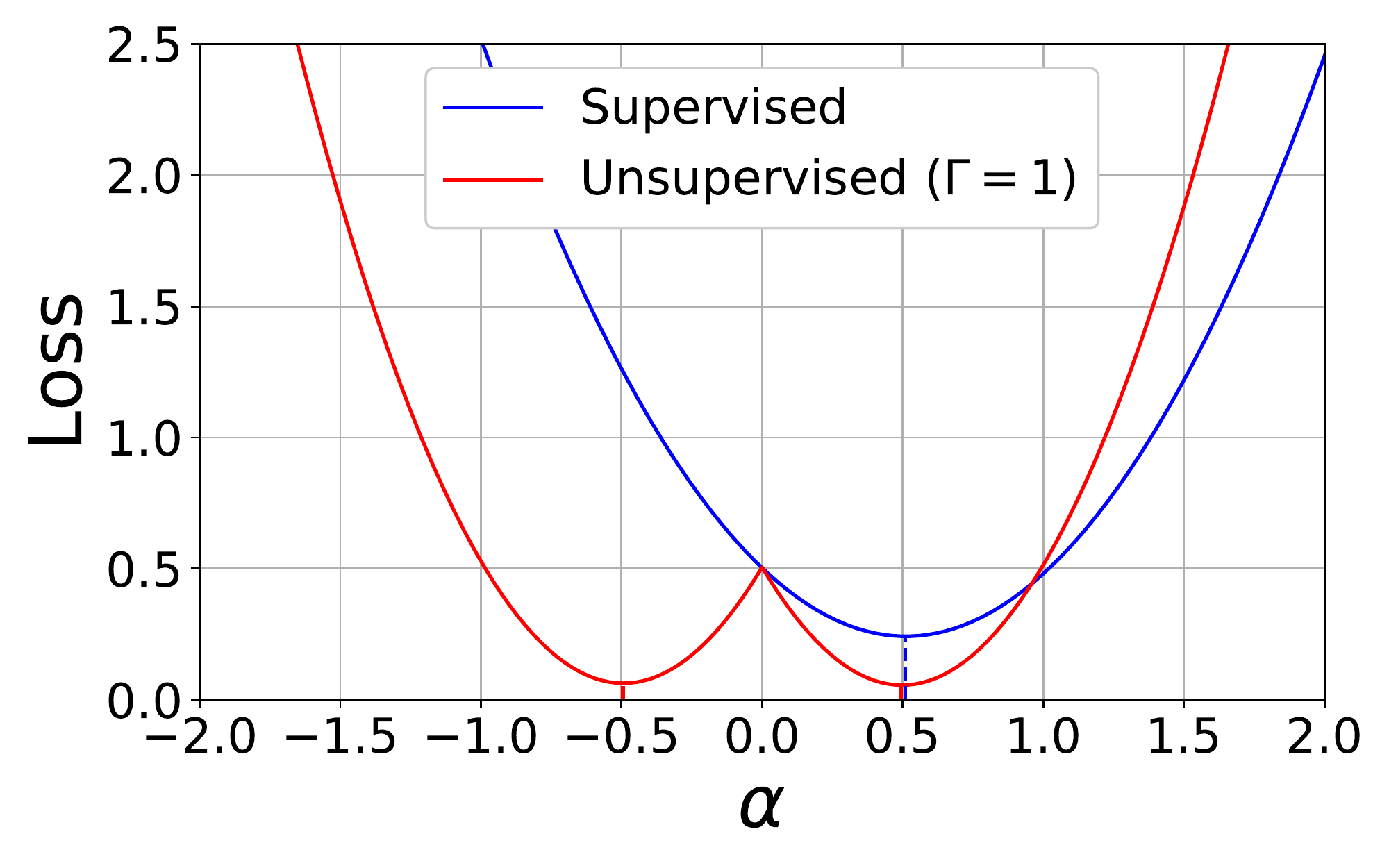}\caption{Same as Figure \ref{fig1} however self-training acceptance threshold is $\Gamma=1$}\label{fig_2g}
	\end{subfigure}~~~
	\begin{subfigure}{2.2in}
		\includegraphics[scale=0.27]{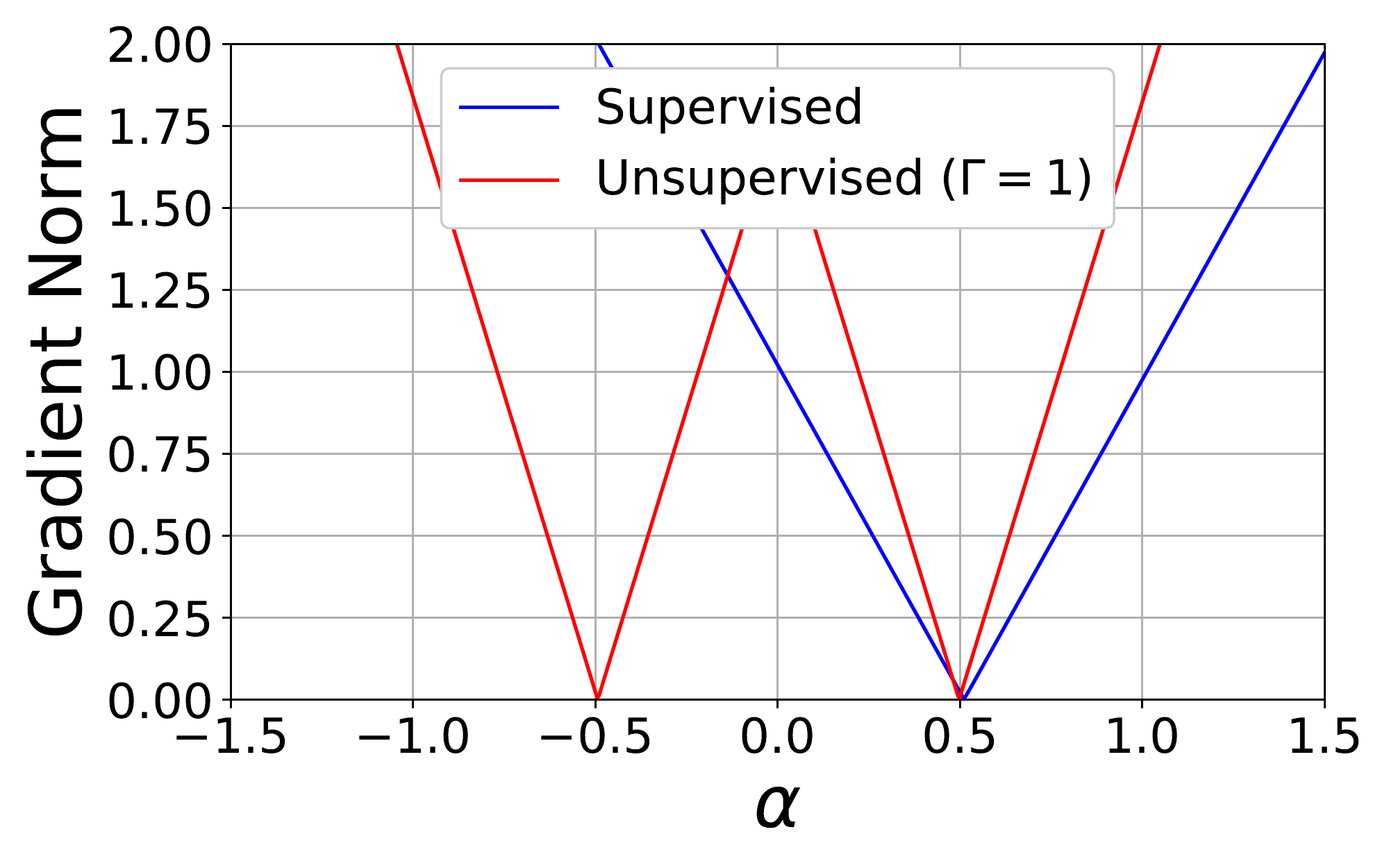}\caption{The gradient lengths in Figure \ref{fig_2g} where $\Gamma=1$.}\label{fig_3g}
	\end{subfigure}
	\caption{Visualizing gradient sizes and the effect of acceptance threshold.}
	\label{fig_gradient} \vspace{-0.2cm}
\end{figure}

Figure \ref{fig2} plots the supervised loss curve along with the semisupervised loss function
of 20\% labeled data which can be expressed as $0.2\Lc(f)+0.8\Lcth(f)$. Comparing Figure \ref{fig2} with Figure \ref{fig1}, we see that the semisupervised loss still has two local minima but only have a unique global minima. This global minima coincide with the global minima of the supervised loss i.e.~both are around $\alpha\approx 0.5$. In summary, as we introduce 20\% labeled data into the loss expression of the unsupervised case, the landscape difference between the semisupervised loss and the supervised loss becomes smaller. In general, semisupervised loss is not symmetrical around $\alpha=0$ as soon as we have some fraction of labeled data in the loss and it has a unique global minima obeying $\alpha>0$. Figure \ref{fig3} shows the landscape of the constrained formulation \eqref{const form}. Here, the critical takeaway is that unsupervised loss has two global minima however constrained $\Xi$ greatly narrows down the search space. Specifically, for semisupervised loss, the regions $\alpha<0$ and $\alpha>0$ are easily distinguishable. We use this intuition to formalize the benefit of weak supervision in Theorem \ref{hetero data thm}. Figure \ref{fig_1g} shows the gradient norms associated with the supervised and unsupervised losses in Figure \ref{fig1} which shows the multiple local minima behavior for unsupervised loss \eqref{unsup loss}. Figure \ref{fig_2g} is same as \ref{fig1} however we choose $\Gamma=1$. Interestingly, the global minima of the unsupervised loss over $\alpha>0$ axis coincides much better with the global minima of supervised loss. This shows how acceptance threshold can improve the loss landscape and the compatibility between actual labels and pseudo-labels. Figure \ref{fig_3g} is the gradient norms corresponding to Figure \ref{fig_2g} which shows that gradients of the unsupervised loss behave sharper compared to Fig \ref{fig_1g}\footnote{This is likely due to the change in the distribution of the data after rejecting weak samples (which have small norms along $\mu$ direction).} and the global minima over $\alpha>0$ has a better match to supervised gradients.




%
%

\vspace{5pt}
\noindent {\bf{Theoretical anaysis:}} Following this intuition, we consider a constrained empirical risk minimization which first constrains the solution space to a smaller set of functions that achieve small loss on $\Uc$ and then searches over this smaller set using $\Sc$. Let $\tilde{\ell}$ be the loss function to be used on the $\Uc$ dataset. For instance, if $\Uc$ is unlabeled, $\tilde{\ell}$ can be the loss function with respect to pseudo-labels. Define the empirical loss functions
\[
\Lch(f)=\frac{1}{n}\sum_{i=1}^n \ell(f(\z_i))\quad\text{and}\quad \Lcth(f)=\frac{1}{\nt}\sum_{i=1}^{\nt} \tell(f(\zt_i)).
\]
We then solve the constrained problem
\begin{align}
\hat{f}=\arg\min_{f\in \Fc}\Lch(f)\quad\text{subject to}\quad \Lcth(f)\leq \Xi,\label{constrained emp}
\end{align}
where $\Xi>0$ is the hyperparameter governing the strength of the constraint.

\noindent {\bf{Landscape compatibility:}} To formalize our analysis we need to characterize how weak-supervision $\Uc$ can help towards finding a solution for the population risk $\Lc$. Following our earlier discussion on semisupervised loss landscape, intuitively, this could be achieved by relating the loss landscapes associated with $\Uc$ and $\Sc$. 

Let $\Lct$ be the population risk of the $\tilde{\Dc}$ distribution i.e. $\Lct(f)=\E_{\zt\sim\tilde{\Dc}}[\tell(f(\zt))]$. We would like to ensure that the loss landscapes of $\Lc$ and $\Lct$ have commonality to a certain extent. The basic idea is that there should be $f\in\Fc$ which achieves small population loss in both objectives. The following definition connects the sublevel sets of both loss functions and will be helpful for formalizing this commonality.
\begin{definition}[Sublevel set and loss commonality] \label{sublevel def}Given $\eps>0$, function class $\Fc$ and loss function $\Lc$, the $\eps$-sublevel set of $\Lc$ is defined as
\[
\Fc_{\Lc,\eps}=\{f\bgl f\in\Fc\quad\text{and}\quad \Lc(f)\leq \min_{f\in\Fc}\Lc(f)+\eps\}.
\]
Given another loss function $\Lct$, let $\epst=\epst(\Lc,\Lct,\eps)$ be the smallest number such that
\begin{align}
\Fc_{\Lc,\eps}\cap \Fc_{\Lct,\epst}\neq \emptyset.\label{non empty}
\end{align} 
\end{definition}
In light of this definition, weak-supervision $\tilde{\Dc}$ would help $\Dc$ when the sublevel sets of the loss functions $\Lc$ and $\Lct$ intersects. This intuition is visualized in  Figure \ref{fig3}. Related notions of compatibility are used in earlier works \cite{balcan2010discriminative,darnstadt2013unlabeled,zhou2014semi} for semi-supervised learning. The following theorem establishes a statistical learning bound based on Rademacher complexity analysis by building on this intuition. 
\begin{theorem}[Learning with Weak-Supervision] \label{hetero data thm} Fix $\eps>0$ and let $\epst$ be as in Definition \ref{sublevel def}. Choose the constraint hyperparameter in \eqref{constrained emp} to be $\Xi= \bXi+\min_{f\in \Fc}\Lct(f)$ with $\bXi\geq 2\epst$. Draw datasets $\Uc=(\zt_i)_{i=1}^u\distas\tilde{\Dc}$ and $\Sc=(\z_i)_{i=1}^n\distas\Dc$. Assume $\ell,\tell:\R\rightarrow[0,1]$ are $L$-Lipschitz loss functions. Suppose sample sizes $n$ and $u$ (are sufficiently large to) satisfy the following Rademacher complexity bounds
\begin{align}
\underbrace{2L\Rc^{\tilde{\Dc}}_{u}(\Fc)+\frac{t}{\sqrt{u}}}_{\text{weakly supervised}}\leq \epst\quad\text{and}\quad \underbrace{2L\Rc^{\Dc}_{n}(\Fc_{\Lct,2\bXi})+\frac{t}{\sqrt{n}}}_{\text{strongly supervised}} \leq \eps.
\end{align}
Then, with probability $1-4\e^{-t^2}$, the solution $\hat{f}$ of the constrained problem \eqref{constrained emp} satisfies
\[
\Lc(\hat{f})\leq \min_{f\in\Fc}\Lc(f)+3\eps.
\]
\end{theorem}
This theorem shows that as long as weak supervision has enough samples to narrow down the initial large search space to a small sublevel set, strong supervision can provably find a generalizing solution with very few samples where the sample complexity is only dictated by the Rademacher complexity of the sublevel set $\Fc_{\Lct,2\bXi}$. Recall that if the strong supervision loss $\Lc$ and the weak supervision loss (e.g.~unsupervised self-training) $\Lct$ have similar sub-level sets, then, $\epst$ can be chosen to be very small which leads to a small search space for supervised loss in \eqref{constrained emp} and only few labels are sufficient for generalization. In essence, the technical idea (loss commonality) of this theorem is inspired from \cite{balcan2010discriminative} which focuses on semi-supervised learning, however we show that the landscape compatibility can shed light on the statistical analysis of the more general problem class of weakly-supervised learning involving heterogeneous datasets.

\vs\vs\section{Conclusions}\vs\vs

In this work, we analyzed the performance of self-training for linear classifiers and mixture distributions. We analytically showed that self-training process would converge to useful solutions for linear classifier parameters in the case of GMM. The theoretical findings demonstrate the benefits of rejecting samples with low-confidence and applying multiple self-training iterations and provides a framework for contrasting various algorithmic choices (e.g.~fresh samples vs reusing samples). We also considered a variation of GMM which reveals that: (1) class margin (in terms of distance between mixture means) is critical for convergence of self-training to useful models and (2) ridge-regularization and early-stopping can enable self-training to converge to good models, in a similar fashion to power iteration converging to principal eigenvector, even without margin requirements. Finally, we discussed the connections between semisupervised learning and learning with weak-supervision and heterogeneous data from a statistical learning perspective. There are many interesting future works especially along joint statistical and algorithmic analysis of more practical self-training problems. It would be of interest to develop non-asymptotic bounds for iterative self-training schemes for more complex distributions and classifiers (e.g.~logistic regression), adapting our approach to multiclass classification, and investigating the self-training behavior for nonlinear models such as deep nets.


\remove{
\section*{Broader Impact}
The recent successful results that large and deep neural networks have mostly depend on the usage of labeled datasets. But such labeled datasets may not be easily accessible for some applications of practical interest.
But unlabeled data could be easier to obtain and labeling them might be costly, meaning that
it could be more convenient to use semi/unsupervised learning techniques to train the models. We still have not developed a full understanding 
for the success of self-training algorithms, or for its fundamental limits.
This work enables us to address the question of
why and how self-training algorithms work by proving some insightful results for relatively simpler settings.
We have characterized some conditions and test scenarios under which self-training approaches provide useful solutions. 
It is crucial to identify some key concepts, such as margin and regularization, for the convergence of semi-supervised learning
to true set of parameters, and show their importance in a rigorous fashion.
The way we have treated the self-training problem here may not be so easily applicable to complex or deep networks.  
Nevertheless, we believe the theoretical results presented here will pave the way for a fruitful line of self-training related research which hopefully will shed more light on the dynamics of semi-supervised learning.}

\small{
\bibliographystyle{acm}
\bibliography{Bibfiles}
}

\newpage

\appendix

\section{Proofs for Section \ref{sec gauss}}
The following lemma provides a straightforward guarantee on the estimation of $\bmu$.
\begin{lemma}[Simple supervised estimator] \label{simple super}Suppose we have $n$ i.i.d.~labeled examples $(\x_i,y_i)_{i=1}^n$ from the GMM model. Consider the supervised estimator
\[
\bti=\frac{1}{n}\sum_{i=1}^n \x_iy_i.
\]
With probability $1-2\e^{-\eps^2p/2}-2\e^{-\eps^2n/2}$, we have that
\[
\frac{1+\sigma\eps}{(1-\eps)_+\sigma\sqrt{p/n}}\geq\tang{\bti,\bmu}\geq \frac{1-\sigma\eps}{(1+\eps)\sigma\sqrt{p/n}}.
\]
\end{lemma}
\begin{proof} Observe that, $\bti$ is distributed as
\[
\bti=\bmu+\h\where \h\sim \Nn(0,\frac{\sigma^2\Iden_p}{n}).
\]
Next, writing $\h=h\bmu+\h^\perp$ where $h\sim \Nn(0,1)$, we have that $|h|\leq \sigma\eps$ with probability at least $1-2\e^{-\eps^2n/2}$ and $(1-\eps)\sigma\sqrt{p/n}\leq \tn{\h^\perp}\leq (1+\eps)\sigma\sqrt{p/n}$ with probability at least $1-2\e^{-\eps^2p/2}$. Combining, we find
\[
\frac{1+\sigma\eps}{(1-\eps)\sigma\sqrt{p/n}}\geq \frac{1+|h|}{\tn{\h^\perp}}\geq \tang{\bti,\bmu}\geq \frac{1-|h|}{\tn{\h^\perp}}\geq \frac{1-\sigma\eps}{(1+\eps)\sigma\sqrt{p/n}}.
\]
\end{proof}

\begin{lemma} \label{split lem}Let $g,h\sim\Nn(0,1)$ and $f:\R\rightarrow \R$ be a bounded function. Then
\[
\E[(h+\sigma g)g]=\sigma\E[(h+\sigma g)h].
\]
\end{lemma}
\begin{proof} Let $z=\sigma h-g$ and observe that $z$ is independent of $h+\sigma g$. Thus, we note that
\[
\E[f(h+\sigma g)g]=\E[f(h+\sigma g)(\sigma h-z)]=\E[f(h+\sigma g)\sigma h]=\sigma\E[f(h+\sigma g) h],
\]
which is the desired statement.
\end{proof}

\subsection{Proof of Theorem \ref{sharp_bound_GMM}}
\begin{proof} Due to symmetry of the input clusters around $0$, without losing generality, we can assume samples belong to the $+$ cluster i.e.~$\x_i\sim \Nn(\bmu,\Iden_p)$. Set $\beta=\sqrt{1-\alpha^2}$. Let us assume $\bti=\alpha \bmu+\beta\vb$ for some unit norm $\vb$ orthogonal to $\bmu$ and analyze $\bth$. Decompose the Gaussian noise vector $\g$ as follows
\[
\g=g_0\bmu+g\vb+\g^\perp.
\]
Here $g_0,g\sim\Nn(0,1)$ and $\g^\perp\sim \Nn(0,\Iden_p-\bmu\bmu^T-\vb\vb^T$). Additionally set $h=\bti^T\g=\alpha g_0+\beta g$. Let $\g_i,g_i,h_i,g_{0,i},\g^\perp_i$ denote the variables associated with the $i$th sample. Proceeding, note that%
\begin{align*}
\Pro(|\bti^T\x|\geq \Gamma)&=\Pro(|\alpha+\sigma h|\geq \Gamma)=\Pro(\{h\geq \bar{\Gamma}_+\}\cup \{h\leq -\bar{\Gamma}_-\})\\
&=Q(\bar{\Gamma}_+)+Q(\bar{\Gamma}_-)\\
&:=\rho
\end{align*}
Set $s=\sum_{i=1}^u 1(|\bti^T\x_i|\geq \Gamma)$. Additionally define $E_i$ to be the event that the pseudo-label prediction is wrong on the $i$th sample i.e.~$E_i=\{|\bti^T\x_i|\geq \Gamma\}\bigcap \{\sgn{\bti^T\x_i}\neq y_i\}$. Similar to above
\[
\Pro(E_i)=\Pro(\alpha+\sigma h\leq -\Gamma)=Q(\bar{\Gamma}_-).
\]
Define $f=\sum_{i=1}^u 1(E_i)$ (the total number of accepted examples with wrong pseudo-label predictions). Chernoff bound yields that with probability $1-4\e^{-\frac{\eps^2\rho u}{3}}$, $s$ and $f$ obeys
\begin{align}
|f-uQ(\bar{\Gamma}_-)|\leq \eps u\rho \quad\text{and}\quad |s-u\rho|\leq \eps u\rho.\label{s bound}
\end{align}
Define the conditional distribution $\xp\sim \x\bgl |\bti^T\x|\geq \Gamma$. Let $\{\xp_i\}_{i=1}^s$ be the $s$ accepted instances out of $u$ (i.e.~$|\bti^T\xp_i|\geq \Gamma$) with this distribution and write $\xp_i=\bmu+\sigma\g'_i$. Then, we can decompose $\g'_i=\g^\perp_i+g'_i\vb+g'_{0,i}\bmu$ where
\[
\g^\perp_i \sim \Nn(0,\Iden-\bmu\bmu^T-\vb\vb^T)\quad\text{and}\quad |\alpha+\sigma h'_i|\geq \Gamma.
\]
and $h'_i:=\beta g'_i+\alpha g'_{0,i}$. This implies $h'_i$ is distributed as $h'\sim h\bgl \{h\geq \bar{\Gamma}_+\}\cup \{h\leq -\bar{\Gamma}_-\}$ where $h\sim\Nn(0,1)$. Without losing generality, suppose $(\xp_i)_{i=1}^f$ are instances with wrong pseudo-label prediction and the rest are instances with correct pseudo-label prediction. To proceed, we estimate $\bth$ as
\begin{align}
\bth&=\frac{1}{s}\sum_{i=1}^s \sgn{\bti^T\xp_i}\xp_i\nn\\
&=\frac{1}{s}\sum_{i=1}^s\sgn{\alpha+\sigma h'_i}(\bmu+\sigma\g^\perp_i+\sigma g'_i\vb+\sigma g'_{0,i}\bmu)\nn\\
&=\underbrace{(1-2\frac{f}{s})\bmu}_{a_1\bmu}+ \underbrace{\frac{\sigma}{s}\sum_{i=1}^s\sgn{\alpha+\sigma h'_i}\g^\perp_i}_{\bt^\perp}+\underbrace{\frac{\sigma}{s}\sum_{i=1}^s\sgn{\alpha+\sigma \alpha g'_{0,i}+\sigma \beta g'_i} (g'_i\vb+g'_{0,i}\bmu)}_{a_2\bmu+a_3\vb}.\label{scalar a123}
\end{align}
where $\bt^\perp$ is orthogonal to $\vb,\bmu$. Using $0\leq \eps\leq 1/2$ and recalling $Q(\bar{\Gamma}_-)/\rho=\nu$, the scalar $a_1$ can be bounded as
\begin{align}
 1-2\nu+8\eps\geq 1-\frac{2Q(\bar{\Gamma}_-)-2\rho\eps}{\rho(1+\eps)}\geq  a_1=1-\frac{2f}{s}\geq 1-\frac{2Q(\bar{\Gamma}_-)+2\rho\eps}{\rho(1-\eps)}\geq 1-2\nu-8\eps.\label{a1 bound}
\end{align}
The $\bt^\perp$ term can be bounded by noting that $\g^\perp$ is independent of $h'_i$ which implies
\[
\bt^\perp=\frac{\sigma}{s}\sum_{i=1}^s\sgn{1+\sigma h'_i}\g^\perp_i\sim \Nn(0,\sigma^2\frac{\Iden-\bmu\bmu^T-\vb\vb^T}{s}).
\]
Consequently, using Gaussianity, $\bt^\perp$ obeys $\sigma \sqrt{\gamma_{p-2}/s}= \E[\tn{\bt^\perp}]$. Via Lipschitz concentration and \eqref{s bound}, this implies with probability at least $1-2\e^{\eps^2 \gamma_{p-2}^2/2}$,
\begin{align}
&(1+\eps)\sigma \gamma_{p-2}/\sqrt{s}\geq \tn{\bt^\perp}\geq (1-\eps)\sigma \gamma_{p-2}/\sqrt{s}\implies \nn\\
&(1+3\eps)\sigma \sqrt{\gamma_{p-2}/u\rho}\geq \tn{\bt^\perp}\geq (1-3\eps)\sigma  \sqrt{\gamma_{p-2}/u\rho}.\label{perp bound}
\end{align}
Finally, what remains is bounding the scalars $a_2$ and $a_3$ in \eqref{scalar a123}. We accomplish this by going back to the original problem rather than the conditional sum which allows us to use Gaussianity. Specifically, we consider the summation
\begin{align}
\frac{1}{u}\sum_{i=1}^u1(|\alpha+\sigma h_i|\geq \Gamma)\sgn{\alpha+\sigma h_i} (g_i\vb+g_{0,i}\bmu):=a_2\bmu+a_3\vb.\label{s23}
\end{align}
First, note that we have the following expectation over $h$
\begin{align*}
\E[1(|\alpha+\sigma h|\geq \Gamma)\sgn{\alpha+\sigma h}h]&=\rho^{-1}\int_{\bar{\Gamma}_+}^\infty \frac{1}{\sqrt{2\pi}}x\e^{-x^2/2}dx+\rho^{-1}\int_{\bar{\Gamma}_-}^\infty \frac{1}{\sqrt{2\pi}}x\e^{-x^2/2}dx\\
&=\frac{1}{\sqrt{2\pi}\rho}(\e^{-\bar{\Gamma}_+^2/2}+\e^{-\bar{\Gamma}_-^2/2}):=\Lambda.
\end{align*}
Now, using Lemma \ref{split lem} and $h=\alpha g_0+\beta g$, this expectation will be proportionally split between the $\bmu$ associated variable $g_0$ and $\vb$ associated variable $g$. Specifically, we have
\begin{align}
&\E[1(|\alpha+\sigma h|\geq \Gamma)\sgn{\alpha+\sigma h}g_0]=\frac{\alpha}{\beta}\E[1(|\alpha+\sigma h|\geq \Gamma)\sgn{\alpha+\sigma h}g]\nn\\
&\alpha\E[1(|\alpha+\sigma h|\geq \Gamma)\sgn{\alpha+\sigma h}g_0]+\beta \E[1(|\alpha+\sigma h|\geq \Gamma)\sgn{\alpha+\sigma h}g]=\Lambda.\nn
\end{align}
This implies 
\begin{align}
&\E[a_2]=\E[1(|\alpha+\sigma h|\geq \Gamma)\sgn{\alpha+\sigma h}g_0]=\alpha\sigma\Lambda\nn\\
&\E[a_3]=\E[1(|\alpha+\sigma h|\geq \Gamma)\sgn{\alpha+\sigma h}g]=\beta\sigma\Lambda.\label{alpha beta eq}
\end{align}
To proceed, we need to show concentration of the average in \eqref{s23} which can be accomplished by noticing the subgaussianities
\[
\tsub{1(|\alpha+\sigma h|\geq \Gamma)\sgn{\alpha+\sigma h}g_0},\tsub{1(|\alpha+\sigma h|\geq \Gamma)\sgn{\alpha+\sigma h}g}\lesssim 1.
\]
These immediately follow from the bounded moments. The subgaussian concentration implies that, with probability at least $1-4\e^{-cu\eps^2}$, we have that
\begin{align}
&\max(|a_2-\E[a_2]|, |a_3-\E[a_3]|)\leq \sigma\eps.
\end{align}
Combining all of the estimates \eqref{a1 bound}, \eqref{perp bound}, \eqref{alpha beta eq}, with the advertised probability, we can write
\[
\bth=a_0\bmu+a_3\vb+\bt^\perp,
\]
where $a_0=a_1+a_2$ and the components satisfy the following two sided bounds
\begin{align*}
&|a_0-(1+\sigma\alpha \Lambda-2\nu)|\leq 8\eps+\sigma \eps\\
&|a_3-\sigma\beta \Lambda|\leq \sigma\eps\\
&|\tn{\bt^\perp}-\sigma\sqrt{\gamma_{p-2}/u\rho}|\leq 3\eps \sigma \sqrt{\gamma_{p-2}/u\rho}.
\end{align*}
This implies that, with probability at least $1-4\e^{-cu\eps^2}-4\e^{-{\eps^2\rho u}/{3}}-2\e^{\eps^2 (p-3)/2}$, we obtain the advertised bound of  
\[
\frac{1+\sigma\alpha \Lambda-2\nu+(8+\sigma)\eps}{\sigma\sqrt{(\beta \Lambda-\eps)_+^2+(1-3\eps)_+^2\gamma_{p-2}/u\rho}} \geq \tang{\bth,\bmu}\geq \frac{1+\sigma\alpha \Lambda-2\nu-(8+\sigma)\eps}{\sigma\sqrt{(\beta \Lambda+\eps)^2+(1+3\eps)^2\gamma_{p-2}/u\rho}}
\]
After mapping $\eps\leftrightarrow \eps/8$, we find that with the advertised probability, we have that
\[
\frac{1+\sigma\alpha \Lambda-2\nu+(1+\sigma)\eps}{\sigma\sqrt{(\beta \Lambda-\eps)_+^2+(1-\eps)_+^2\gamma_{p-2}/u\rho}} \geq \tang{\bth,\bmu}\geq \frac{1+\sigma\alpha \Lambda-2\nu-(1+\sigma)\eps}{\sigma\sqrt{(\beta \Lambda+\eps)^2+(1+\eps)^2\gamma_{p-2}/u\rho}}.
\]
Convergence in probability immediately follows from this non-asymptotic bound.
\end{proof}

\section{Proofs for Section \ref{sec algo}}
Throughout, we assume $\bti$ is unit Euclidian norm without losing generality. This is to simplify the subsequent notation.
\subsection{Proof of Theorem \ref{thm no}}
\begin{proof} Let us recall the distribution of the data. Given label $y$, we have that $\x=g\bmu+\sigma\g$ where $g:=yX$. Noticing $g\sim\Nn(0,1)$, this means that the marginal distribution of $\x$ is $\Nn(0,\bSi)$ where covariance matrix is $\bSi=\sigma^2\Iden+\bmu\bmu^T$. Thus, our interest is understanding the minimizer $\bth$ of \eqref{PL sup}. For this, we have the following lemma that applies to arbitrary covariance matrices.
\begin{lemma} Let $\x\sim\Nn(0,\bSi)$ where $\bSi$ is a full-rank positive-semidefinite matrix and set $\Gamma\geq0$. Then, the minimizer of \eqref{PL sup} obeys 
\[
\bth=\frac{c_\Gamma\bti}{\tn{\sqrt{\bSi}\bti}}.
\]
For the special case of $\Gamma=0$, $c_\Gamma=\sqrt{2/\pi}$.
\end{lemma}
\begin{proof} Let $\xp\sim \x\bgl |\x^T\bti|\geq \Gamma$. Let $\bSi'$ be the covariance of $\xp$. Let us differentiate the loss with respect to $\bt$. This yields
\[
\E[\sgn{\bti^T\xp}\xp]=\E[\xp\xp^T]\hat{\bt}\implies \hat{\bt}=\bSi'^{-1}\E[\sgn{\bti^T\xp}\xp].
\]
Let us write $\x=\sqrt{\bSi}\xb$ so that $\xb\sim\Nn(0,\Iden)$. Set $\ab=\sqrt{\bSi}\bti$ and $\abb=\ab/\tn{\ab}$. Decompose $\xb=\abb h+\xb^\perp$ where $\xb^\perp$ is independent of $h\sim\Nn(0,1)$. Also let $h'\sim h\bgl |h|\geq \Gamma/\tn{\ab}$. With this, we note that $\xp\sim\sqrt{\bSi}\xb^\perp+\sqrt{\bSi}\abb h'$. Consequently, using independence of $h'$ and $\xb^\perp$, we obtain
\begin{align}
\s=\E[\sgn{\bti^T\xp}\xp]=\E[\sgn{h'}\sqrt{\bSi}\abb h']=\sqrt{\bSi}\abb\E[|h'|]=\frac{\bSi\bti}{\tn{\ab}}\E[|h'|].\label{s eqn}
\end{align}
Secondly, observe that $\bSi'=\E[\xp\xp^T]=\sqrt{\bSi}(\Iden+(\E[h'^2]-1)\abb\abb^T)\sqrt{\bSi}$. This yields
\begin{align}
\bSi'^{-1}\s&=\bSi^{-1/2}(\Iden+(\E[h'^2]-1)\abb\abb^T)^{-1}\bSi^{-1/2}\frac{\bSi\bti}{\tn{\ab}}\E[|h'|]\\
&=\bSi^{-1/2}(\Iden+(\E[h'^2]-1)\abb\abb^T)^{-1}\abb\E[|h'|]\\
&=\bSi^{-1/2}\abb\E[h'^2]^{-1}\E[|h'|]\\
&=\frac{\E[|h'|]\bti}{\E[h'^2]\tn{\sqrt{\bSi}\bti}}
\end{align}
For the special case of $\Gamma=0$, $h=h'\sim\Nn(0,1)$ which implies $\E[|h'|]/\E[h'^2]=\sqrt{2/\pi}$.
\end{proof} 
This result also implies that for the original covariance matrix $\bSi=\sigma^2\Iden+\bmu\bmu^T$, $\bth$ will have the same direction as $\bti$ with an additional scaling that depends on problem parameters $\sigma,\bti,\bmu$.
\end{proof}

\subsection{Proof of Theorem \ref{thm margin}}
\begin{proof} 
We will argue that $\bti$ successfully labels a large fraction of the unlabeled data under the margin condition $\gamma>0$. The covariance of the input is again given by 
\[
\bSi=\sigma^2\Iden+\bmu\bmu^T.
\]
The population model is given by
\[
\hat{\bt}=\bSi^{-1}\E[\sgn{\bti^T\x}\x]=\bSi^{-1}\E[\sgn{\bti^T\xp}\xp].
\]
where $\xp=y\x= X+\sigma \g$\footnote{Here we slightly abuse the notation by using $\g\leftrightarrow y\g$ via the rotational invariance of the standard normal $\g$.}. To proceed, we will analyze $\hat{\bt}$ along the $\bmu$ direction and its orthogonal subspace. Without losing generality let us assume $\bti$ is unit length and apply orthogonal decomposition $\bti=\alpha\bmu+\sqrt{1-\alpha^2}\bmu^\perp$ where $\tn{\bmu^\perp}^2=1$. Also decompose $\g=g\bmu+g^\perp\bmu^\perp+\g_r$. We can write
\begin{align}
\bti^T\xp=\alpha (X+\sigma g)+\sigma\sqrt{1-\alpha^2}g^\perp\where g^\perp\sim\Nn(0,1)\label{decomp}
\end{align}
To proceed, we decompose
\[
\bth=\underbrace{\bSi^{-1}\E[\sgn{\bmu^T\xp}\xp]}_{\bth_m}+\underbrace{\bSi^{-1}\E[r(\xp)\xp]}_{\bth_p}
\]
where $r(\xp)=\sgn{\bti^T\xp}-\sgn{\bmu^T\xp}$. The first component precisely returns the supervised model i.e.
\begin{align}
\bth_m&=\bSi^{-1}\E[\sgn{X+\sigma g}((X+\sigma g)\bmu+\g_r)]\\
&=\bSi^{-1}\E[|X+\sigma g|]\bmu\\
&=\frac{\E[|X+\sigma g|]\bmu}{1+\sigma^2}
\end{align}
Next, we focus on the perturbation term. Observe that $\g_r$ is independent of $r(\xp)$ thus, we have that
\begin{align}
\bth_p&=\bSi^{-1}\E[r(\xp)\xp]=\bSi^{-1}\E[r(\xp)((X+\sigma g)\bmu+\sigma g^\perp\bmu^\perp)]\\
&=\frac{\bmu}{1+\sigma^2}\E[r(\xp)(X+\sigma g)]+\frac{\bmu^\perp}{\sigma} \E[r(\xp)g^\perp].
\end{align}
Setting $\hat\alpha=\rho(\bth,\bmu)$, this shows that
\[
\frac{\hat\alpha}{\sqrt{1-\hat\alpha^2}}\geq \frac{\sigma}{1+\sigma^2}\frac{\E[|X+\sigma g|]-\E[|r(\xp)(X+\sigma g)|]}{\E[|r(\xp)g^\perp|]}
\]
Lemma \ref{prop reject} states that $\Pro(r(\xp)\neq 0)\leq 2Q(C)\leq \e^{-C^2/2}$ where $C=\frac{\alpha\gamma}{\sigma}$. Using this, Lemma \ref{lem simple} and $|r(\xp)|\leq 2$, we have the followings.
\begin{itemize}
\item $\E[|X+\sigma g|]\geq \E[X]\geq \gamma$.
\item $\E[|r(\xp)X|]\leq 4\e^{-C^2/2}M\gamma$.
\item $\E[|r(\xp)\sigma g|]\leq 2\sigma \e^{-C^2/2}$.
\item $\E[|r(\xp) g^\perp|]\leq 2 \e^{-C^2/2}$.
\end{itemize}
Plugging these, we find
\[
\tang{\bth,\bmu}=\frac{\hat\alpha}{\sqrt{1-\hat\alpha^2}}\geq \frac{\sigma\e^{C^2/2}}{2(1+\sigma^2)}(\gamma(1- 4\e^{-C^2/2}M)-2\sigma \e^{-C^2/2}).
\]
The advertised results follows from this bound by specializing to $\sigma\leq \gamma$ and $M\geq 1$ and then applying the change of variable $C\leftrightarrow C^2/2$.
\end{proof}

\subsection{Proof of Lemma \ref{lem ridge}}
\begin{proof} The proof is similar to that of Theorem \ref{thm no}. Following same notation as the proof of Theorem \ref{thm no}, setting $\Gamma=0$ and plugging covariance $\bSi$, the solution is given by
\[
\bth=((\la+\sigma^2)\Iden+\bmu\bmu^T)^{-1}\s.
\]
$\s=c(\sigma^2\Iden+\bmu\bmu^T)\bti$ following from \eqref{s eqn}. Writing $\bti=\alpha \bmu+\sqrt{1-\alpha^2}\bmu^\perp$ and noticing $\bmu$ and $\bmu^\perp$ are eigenvectors of $\bSi$, we obtain
\begin{align}
c^{-1}\bSi^{-1}\s&=\frac{1+\sigma^2}{1+\sigma^2+\la}\alpha \bmu+\frac{\sigma^2}{\sigma^2+\la}\sqrt{1-\alpha^2}\bmu^\perp.
\end{align}
This implies the correlation guarantee given by Lemma \ref{lem ridge} noticing the ratio of the $\bmu$ and $\bmu^\perp$ terms above.
\end{proof}

\subsection{Proof of Lemma \ref{lem early}}
\begin{proof} The proof is similar to that of Theorem \ref{thm no}. Following same notation as the proof of Theorem \ref{thm no} and recalling \eqref{s eqn}, we have
\[
\bth=\Pro(|\x^T\bti|\geq \Gamma)\frac{\bSi\bti}{\tn{\ab}}\E[|h'|]=c\bSi\bti.
\]
Writing $\bti=\alpha \bmu+\sqrt{1-\alpha^2}\bmu^\perp$ and noticing $\bmu$ and $\bmu^\perp$ are eigenvectors of $\bSi$, we obtain
\begin{align}
c^{-1}\bth=(1+\sigma^2)\alpha \bmu+\sigma^2\sqrt{1-\alpha^2}\bmu^\perp.
\end{align}
This implies the correlation guarantee given by \eqref{early rat} noticing the ratio of the $\bmu$ and $\bmu^\perp$ terms above.
\end{proof} 

\subsection{Proof of Lemma \ref{prop reject}}

\begin{lemma}[Properties of rejection] \label{prop reject}Fix unit norm vectors $\bti,\bmu\in\R^p$ with $\rho(\bmu,\bti)=\alpha$. Let $\g\sim\Nn(0,\Iden)$ and $X$ be a strictly positive random variable obeying $X\geq \gamma=\sigma\bar{\gamma}>0$. Set $\x=X\bmu+\sigma \g$ and let $\z$ be the random vector with conditional distribution $\x\bgl |\x^T\bti|\geq\Gamma$. We have that
\begin{itemize}
\item When $\Gamma=0$: $\Pro(\{\sgn{\bti^T\z}\neq \sgn{\bmu^T\z}\})\leq 2Q(\alpha \bar{\gamma})$.
\item General $\Gamma>0$: Using the change of variable $\Gamma=\alpha\sigma\Gm$, we have
\[
\Pro(\{\sgn{\bti^T\z}\neq \sgn{\bmu^T\z}\})\leq 2\frac{Q(\bar{\gamma})Q(\frac{\alpha\Gm}{\sqrt{1-\alpha^2}})+Q(\alpha (\bar{\gamma}+\Gm))}{Q_X(\Gm)}.
\]
\end{itemize}
\end{lemma}
\begin{proof} Represent $\g=\sigma g\bmu+\g^\perp$ and set $g'=\li\bmu^\perp,\g^\perp\ri$ where $\bti=\alpha\bmu+\sqrt{1-\alpha^2}\bmu^\perp$. We analyze the event $E=\{\sgn{\bti^T\z}\neq \sgn{\bmu^T\z}\}$. Clearly
\begin{align}
\Pro(E)\leq \underbrace{\Pro(\bti^T\z<0)}_{P(\bti)}+\underbrace{\Pro(\bmu^T\z<0)}_{P(\bmu)}.\label{e formula}
\end{align}
After bounding $\Pro(E)$, we also have that $|\E[\sgn{\bti^T\z}-\sgn{\bmu^T\z}]|\leq 2\Pro(E)$.

\noindent {\bf{When $\Gamma=0$:}} First, using $\bmu^T\x=X+\sigma g$, we bound
\[
P(\bmu)=P(\sigma g<-X)P(g<-\gamma/\sigma)= Q({\gamma}/{\sigma}).
\]
Secondly, recalling \eqref{decomp}, we bound
\[
P(\bti)=P(\sqrt{1-\alpha^2}g'+\alpha g>\alpha X)=Q({\alpha\gamma}/{\sigma})
\]

\noindent {\bf{When $\Gamma>0$:}} For $\Gamma>0$, we condition on the event $A=\{|\x^T\bti|\geq\Gamma\}$ which is equivalent to 
\[
|\alpha X+\sigma h|\geq \Gamma\where h=\alpha g+\sqrt{1-\alpha^2} g'.
\]
Following \eqref{e formula}, we are interested in
\[
\Pro(E)\leq \frac{\Pro(\bti^T\x<0\cap A)}{P(A)}+\frac{\Pro(\bmu^T\x<0\cap A)}{P(A)}.
\]
First, note that
\begin{align}
P(A)&\geq \Pro(\alpha X+\sigma h\geq \Gamma)\\
&\geq \Pro(\alpha X+\sigma h\geq \Gamma\bgl h>0)\Pro(h>0)\geq \frac{1}{2}\Pro(\alpha X\geq \Gamma)\\
&=\frac{Q_X(\Gamma/\alpha)}{2}
\end{align}
Secondly, we have
\[
P(A)\Pro(\bti)\leq \Pro(\alpha X+\sigma h\leq -\Gamma)=\Pro(h\leq \frac{-\alpha X-\Gamma}{\sigma})\leq Q(\frac{\alpha \gamma+\Gamma}{\sigma})
\]
Finally, we are interested in the probability $\Pro(X+\sigma g<0\cap A)$. Intersection event implies two things
\begin{itemize}
\item $E_1=\{X+\sigma g<0\}$.
\item $E_2=\{g'\geq \frac{\Gamma}{\sigma\sqrt{1-\alpha^2}}\}$. This follows from the fact that $X+\sigma g<0$ and event $A$ as follows
\[
\Gamma\leq \alpha X+\alpha \sigma g+\sqrt{1-\alpha^2}\sigma g'\leq \sqrt{1-\alpha^2}\sigma g'\implies g'\geq \frac{\Gamma}{\sigma\sqrt{1-\alpha^2}}
\]
\end{itemize}
Using independence of $g,g'$, we obtain
\[
\Pro(X+\sigma g<0\cap A)\leq \Pro(E_1\cap E_2)= \Pro(E_1)\Pro(E_2)\leq Q(\frac{\gamma}{\sigma})Q(\frac{\Gamma}{\sigma\sqrt{1-\alpha^2}})
\]
which upper bounds $\Pro(A)\Pro(\bmu)$. Combining these, we obtain the desired conclusion.
\end{proof}

\subsection{Proof of Lemma \ref{lem simple}}

\begin{lemma}\label{lem simple} Let $g\sim\Nn(0,1)$ and $E$ be an event with probability $\Pro(E)=Q'(\alpha)=2Q(\alpha)$ where $Q'$ is the tail of folded normal distribution. We have that 
\[
\E[1(E)|g|]\leq\sqrt{2/\pi}\e^{-\alpha^2/2}.
\]
\end{lemma}
\begin{proof} Let $f$ be the density function of folded normal. Observe that
\begin{align*}
\E[1(E)|g|]&=\int_0^\infty \Pro(\{|g|>x\}\cap E) dx\\
&=\int_0^\alpha \Pro(E) dx+\int_\alpha^\infty Q'(x) dx\\
&=Q'(\alpha)\alpha+\int_\alpha^\infty Q'(x) dx\\
&=\int_\alpha^\infty xf(x)dx\\
&=\int_{\alpha}^\infty \sqrt{2/\pi}x\e^{-x^2/2}dx\\
&=\sqrt{2/\pi}\e^{-\alpha^2/2}.
\end{align*}
\end{proof}

\subsection{Proof of Lemma \ref{lem iter bound}}
\begin{proof} Lemma \ref{simple super} shows that asymptotically $\tang{\bti,\bmu}\conv \sqrt{\bar{n}}/\sigma$. Since $(\Uc_i)_{i=1}^\tau$ are disjoint subsets, $\bt_i$ is independent of $\Uc_{i+1}$ and each iteration of self-training will apply $F_{\bar{u}}$ function on the co-tangent of the current iterate as a consequence of Theorem \ref{sharp_bound_GMM}. This leads to the advertised bound.
\end{proof}
\subsection{Proof of Lemma \ref{lem simple stuff}}

\begin{proof} Using right-continuity of cumulative distribution function, for any $\eps>0$, there exists $\delta>0$ such that $\Pro_{\Dc}(f(\x)\leq\delta)\geq 1-\eps$. Thus, writing the expected loss as an integral over inputs $f(\x)<\delta$ and $f(\x)\geq \delta$,
\[
\Lct(\alpha f)\leq \eps \ell(0)+(1-\eps) \ell(\alpha \delta),
\]
which implies $\lim_{\alpha\rightarrow \infty}\Lct(\alpha f)\leq \eps \ell(0)$. Since this is true for any $\eps\geq 0$, the limit is zero.
\end{proof}

\section{Proofs for Section \ref{sec hetero}}

\subsection{Proof of Lemma \ref{unsup gen lemma}}
\begin{proof} 
Define the Rademacher complexity of the composition
\[
\Rc_u(\ell\odot\Fc)=\frac{1}{u}\E[\sup_{f\in\Fc}\sum_{i=1}^u \eps_i\ell_\gamma(|f(\x_i)|)].
\]
$\ell(|x|)$ is $\gamma^{-1}$-Lipschitz function of $x$, hence Rademacher contraction inequality yields
\[
\Rc_u(\ell\odot\Fc)\leq \gamma^{-1}\Rc_u(\Fc).
\]
To proceed, applying standard generalization bound, with probability $1-\delta/2$ over the samples, for all $f\in \Fc$, we have that
\begin{align}
\E_{\Dc}[\ell_\gamma(|f(\x)|)] \leq \frac{1}{u}\sum_{i=1}^u\ell_\gamma(|f(\x_i)|) + \frac{2}{\gamma}\Rc_u(\Fc)+\sqrt{\frac{\log(2/\delta)}{u}}.
\end{align}
Let $\Lc_U^*=\min\E_{\Dc}[\ell_\gamma(|f(\x)|)]$ and $f^*=\arg\min\E_{\Dc}[\ell_\gamma(|f(\x)|)]$. With probability $1-\delta/2$, $f^*$ satisfies
\[
\frac{1}{u}\sum_{i=1}^u\ell_\gamma(|f^*(\x_i)|) \leq \Lc_U^*+\sqrt{\frac{\log(2/\delta)}{u}}.
\]
Combining these two estimates and using optimality of $\hat{f}$, with probability at least $1-\delta$, 
\[
\E[\ell_\gamma(|\hat{f}(\x)|)]\leq \Lc_U^*+  \frac{2}{\gamma}\Rc_u(\Fc)+2\sqrt{\frac{\log(2/\delta)}{u}}.
\]
Noticing $\Lc_U^*\leq \min_{f\in\Fc}\Pro(|f(\x)|\leq 2\gamma)$ and $\E[\ell_\gamma(|\hat{f}(\x)|)]\geq \Pro(|\hat{f}(\x)|\leq \gamma)$ concludes the proof.
\end{proof}

\subsection{Proof of Theorem \ref{hetero data thm}}
\subsubsection{Deterministic analysis}
Following the setup of Theorem \ref{hetero data thm}, in this section, we consider the deterministic conditions on the loss landscape that guarantees favorable properties of the constrained problem \eqref{constrained emp}. Specifically, we make the following assumption that connects the landscape of empirical risk to the population risk.


\begin{assumption}[Empirical is close to population] \label{assume close} Fix scalars $\eps>0$, $\delta>0$, $\bXi\geq \epst+\delta$. Define the sublevel set $\Fc'=\Fc_{\Lct,\bXi+\delta}$. The loss landscape of strong and weak supervision satisfy the following bounds.
\begin{itemize}
\item $\max_{f\in\Fc}|\Lct(f)-\Lcth(f)|\leq \delta$.
\item $\max_{f\in\Fc'}|\Lc(f)-\Lch(f)|\leq \eps$.
\end{itemize}
\end{assumption}
Under this assumption, we have the following guarantee for the solution of the constrained empirical risk problem.
\begin{theorem} \label{thm deter}Suppose Assumption \ref{assume close} holds for an $(\eps,\delta)$ pair. Then, the solution to \eqref{constrained emp} with choice $\Xi=\bXi+\min_{f\in\Fc}\Lct(f)$ with $\bXi\geq \epst+\delta$ satisfies 
\[
\Lc(\hat{f})\leq \Lc^\star+3\eps.
\]
\end{theorem}
\begin{proof} Observe that, the constraint set of our problem is the sublevel set $\Fc_{\Lcth,\bXi}=\{f\in\Fc\bgl \Lcth(f)\leq \Xi\}$. The first statement of Assumption \ref{assume close} implies that the sublevel sets with respect to $\Lcth$ can be bounded via
\[
\Fc_{\Lct,\alpha}\subseteq \Fc_{\Lcth,\alpha+\delta}\subseteq\Fc_{\Lct,\alpha+2\delta}\quad\text{for all}\quad \alpha\geq 0
\]
Consequently, using $\bXi\geq \epst+\delta$. we find that
\[
\Fc_{\Lct,\epst}\subset\Fc_{\Lcth,\bXi}\subset \Fc_{\Lct,\bXi+\delta}.
\]
Following the definition of $\epst$ (i.e.~\eqref{non empty}), this implies that there exists $f'\in\Fc_{\Lcth,\bXi}$ such that $\Lc(f')\leq \Lc^\star+\eps$.

To proceed, the second statement of Assumption \ref{assume close} guarantees that for all $f\in\Fc_{\Lcth,\bXi+\delta}$ (thus for all feasible $f\in \Fc_{\Lcth,\bXi}$) we have that $|\Lc(f)-\Lch(f)|\leq \eps$. Consequently, using the fact that $\hat{f}$ minimizes the empirical risk over the feasible set (which includes $f'$), we find that
\[
\Lc(\hat{f})-\eps\leq \Lch(\hat{f})\leq \Lch(f')\leq \Lc(f')+\eps\leq \Lc^\star+2\eps,
\]
concluding the proof.
\end{proof}

\subsubsection{Finishing the proof (analysis for random data)}
The proof will be concluded by plugging in the necessary sample complexity bounds to guarantee that Assumption \ref{assume close} holds. We pick $\delta=\epst$ and $\bXi\geq 2\epst$ in our bound and in light of Theorem \ref{thm deter}, we would like to guarantee that
\[
\max_{f\in\Fc}|\Lct(f)-\Lcth(f)|\leq \epst\quad\text{and}\quad \max_{f\in\Fc'}|\Lc(f)-\Lch(f)|\leq \eps
\]
where $\Fc_{\Lct,\bXi+\epst}\subseteq \Fc'=\Fc_{\Lct,2\bXi}$ (We remark that, we will ensure the desired bound holds over the larger set $\Fc'$).

Recall that $\ell,\tell:\R\rightarrow[0,1]$ are both $L$ Lipschitz functions. Thus, standard Rademacher complexity based concentration bound \cite{bartlett2002rademacher} implies that, we have that
\begin{align}
\sup_{f\in \Fc}|\Lct(f)-\Lcth(f)|\leq 2L\Rc_u(\Fc)+\frac{t}{\sqrt{u}}\\
\sup_{f\in \Fc'}|\Lc(f)-\Lch(f)|\leq 2L\Rc_n(\Fc')+\frac{t}{\sqrt{n}}.
\end{align}
each with probability at least $1-2\e^{-t^2}$. Thus, if $n$ and $u$ satisfies the advertised bounds, we find that $\sup_{f\in \Fc}|\Lct(f)-\Lcth(f)|\leq \epst$ and $\sup_{f\in \Fc'}|\Lc(f)-\Lch(f)|\leq \eps$. Plugging these in Theorem \ref{thm deter} concludes the proof.

\section{Self-Training with Fresh Samples Can Beat Supervised Learning}
Following the setup of Figure \ref{fig:performance}, we consider the following question: Can Fresh-ST, self-training with $u$ fresh unlabeled data at every iteration, beat supervised learning with $u$ labels? There is no good reason for the answer to be negative however the answer is not clearly visible from Figure \ref{fig3g}. In this section, we zoom into Figure \ref{fig3g} by plotting the accuracy gap between supervised learning and Fresh-ST which is displayed in Figure \ref{fig8}. y-axis shows the accuracy gap $\text{acc(Fresh-ST}(\tau))-\text{acc(supervised)}$ where $\tau$ is the number of iterations. As $\tau$ increases, the accuracy gap achieves positive values proving that Fresh-ST can go beyond supervised learning. Note that this claim is already very visible for logistic regression (see Fig.~\ref{fig2}). This section clarifies this for averaging estimator as well.

\begin{figure}[t!]
\centering
		\includegraphics[scale=0.4]{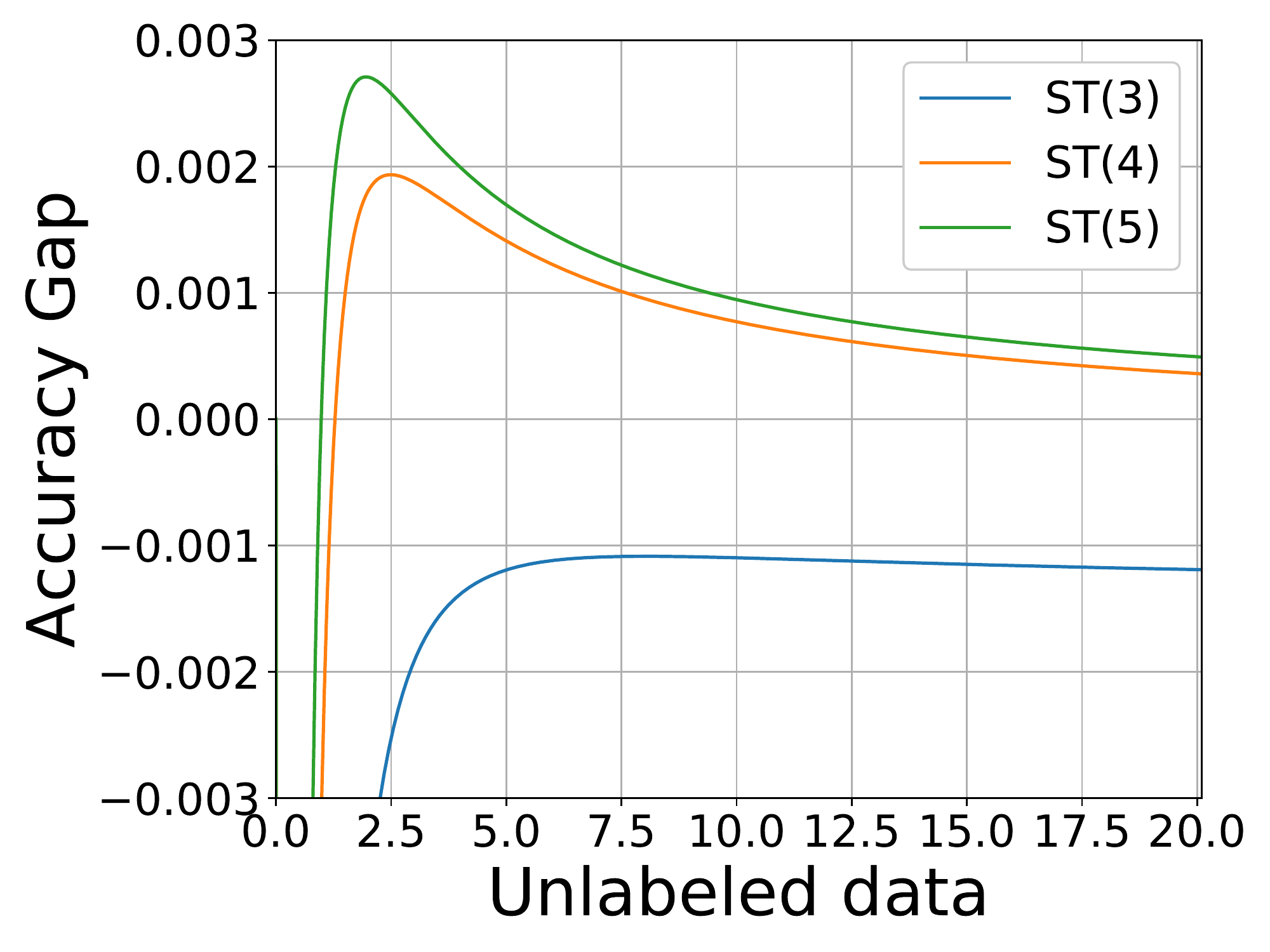}
		\caption{The gap between Fresh-ST (self-training with fresh unlabeled data) and supervised learning. Both uses $u$ samples. This figure shows that supervised learning does not upper bound Fresh-ST and a few self-training iteration can provably go beyond supervised bound. The setup is same as in Figure \ref{fig:performance}.}
	\label{fig8} \vspace{-0.2cm}
\end{figure}

\end{document}